\documentclass[twoside,leqno,twocolumn]{article}  

\newif\ifapx
\apxtrue % uncomment if you want appendix
\usepackage{times}
\usepackage[english]{babel}

\usepackage{ltexpprt}
\usepackage{algorithmic}
\usepackage{algorithm}
\usepackage{amsmath}
\usepackage{amssymb}
\usepackage{verbatim}
\usepackage{multirow}
\usepackage{booktabs}
\usepackage{url}
\usepackage{cite}
\usepackage{xspace}
\usepackage{array}
\usepackage{dcolumn}
\usepackage{graphicx}
\usepackage[tight,center]{subfigure}
\usepackage{color}
\usepackage{makeidx}
\usepackage{verbatim}
\usepackage{caption}
\usepackage{float}
\usepackage{enumitem}
\usepackage{pifont}
\usepackage{enumitem} % for adjusting enum/itemize
\usepackage{calc} % for \widthof

\usepackage{tikz}
\usetikzlibrary{snakes,arrows,shapes,positioning,fit,calc,patterns}
\usepackage{pgfplots}

\newcommand{\otoprule }{\midrule[\heavyrulewidth]}

\newcommand{\xmark}{\ding{55}}%

\newcommand{\ourmethod}{\textsc{flexi}\xspace}
\newcommand{\qr}{\mathit{qr}}
\newcommand{\WRAcc}{\mathit{WRAcc}}
\newcommand{\kl}{\mathit{kl}}
\newcommand{\hd}{\mathit{hd}}
\newcommand{\zscore}{\mathit{z\textit{-}score}}
\newcommand{\id}{\mathit{ID}}

\newcommand{\alter}{\textsc{sum}\xspace}
\newcommand{\smdl}{\textsc{sd}\xspace}
\newcommand{\unml}{\textsc{ud}\xspace}
\newcommand{\ipd}{\textsc{ipd}\xspace}
\newcommand{\RocInt}{\textsc{roc}\xspace}
\newcommand{\ef}{\textsc{ef}\xspace}
\newcommand{\ew}{\textsc{ew}\xspace}

\newcommand{\wflexi}{\textsc{flexi}$_{\mathit{w}}$\xspace}
\newcommand{\zflexi}{\textsc{flexi}$_{\mathit{z}}$\xspace}
\newcommand{\kflexi}{\textsc{flexi}$_{\mathit{k}}$\xspace}
\newcommand{\hflexi}{\textsc{flexi}$_{\mathit{h}}$\xspace}
\newcommand{\qflexi}{\textsc{flexi}$_{\mathit{q}}$\xspace}

\newcommand{\D}{\mathbf{D}}
\newcommand{\da}{\mathbf{d}}

\newcommand{\Ab}{\mathbf{A}}
\newcommand{\A}{A}
\newcommand{\Tb}{\mathbf{T}}
\newcommand{\T}{T}
\newcommand{\size}{m}
\newcommand{\dima}{n}
\newcommand{\dimt}{d}
\newcommand{\Subsett}{\mathcal{S}}
\newcommand{\Sub}{\mathit{S}}
\newcommand{\Subc}{\overline{\Sub}}
\newcommand{\Sube}{\mathit{R}}
\newcommand{\Dsc}{\mathcal{F}}
\newcommand{\diff}{\mathit{div}}

\newcommand{\countPos}{\mathit{countPos}}
\newcommand{\binMean}{\mathit{binMean}}

\newcommand{\pos}{\mathit{pos}}
\newcommand{\dom}{\mathit{dom}}
\newcommand{\qual}{\mathit{qual}}
\newcommand{\inter}{\mathit{int}}

\newcommand{\maxv}{\mathit{V}}
\newcommand{\minv}{\mathit{v}}
\newcommand{\dsc}{\mathit{dsc}}

\newcommand{\proofApx}{
\begin{proof}
\ifapx 
 We postpone the proof to Appendix~\ref{sec:proofs}.
\else 
 We postpone the proof to the online Appendix.
\fi 
\end{proof}
}

\newcommand{\cmark}{\checkmark}
\makeatletter
\renewcommand*{\@fnsymbol}[1]{\ensuremath{\ifcase#1\or   \circ\or \bullet\or *\or \ddagger\or
   \mathsection\or \mathparagraph\or \|\or **\or \dagger\dagger
   \or \ddagger\ddagger \else\@ctrerr\fi}}
\makeatother

\pgfdeclarelayer{background}
\pgfdeclarelayer{foreground}
\pgfsetlayers{background,main,foreground}

\tikzstyle{block} = [rounded corners, draw=blue!70, fill=white, text width=3.3cm, minimum height=4em]
\tikzstyle{bgblock} = [rounded corners, draw=blue!70, thick, fill=blue!10, text width=3.3cm, minimum height=4em]
\tikzstyle{line} = [draw, -latex', thick,blue!70]

\definecolor{yafaxiscolor}{rgb}{0.3, 0.3, 0.3}

\definecolor{yafcolor1}{rgb}{0.4, 0.165, 0.553}
\definecolor{yafcolor2}{rgb}{0.949, 0.482, 0.216}
\definecolor{yafcolor3}{rgb}{0.47, 0.549, 0.306}
\definecolor{yafcolor4}{rgb}{0.925, 0.165, 0.224}
\definecolor{yafcolor5}{rgb}{0.141, 0.345, 0.643}
\definecolor{yafcolor6}{rgb}{0.965, 0.933, 0.267}
\definecolor{yafcolor7}{rgb}{0.627, 0.118, 0.165}
\definecolor{yafcolor8}{rgb}{0.878, 0.475, 0.686}

\newlength{\yafaxispad}
\newlength{\yaftlpad}
\newlength{\yaflabelpad}
\newlength{\yafaxiswidth}
\newlength{\yafticklen}
\makeatletter
\def\pgfplots@drawtickgridlines@INSTALLCLIP@onorientedsurf#1{}
\makeatother

\newcommand{\yafdrawaxes}[4]{
	\pgfplotstransformcoordinatex{#1}\let\xmincoord=\pgfmathresult 
	\pgfplotstransformcoordinatex{#2}\let\xmaxcoord=\pgfmathresult 
	\pgfplotstransformcoordinatey{#3}\let\ymincoord=\pgfmathresult 
	\pgfplotstransformcoordinatey{#4}\let\ymaxcoord=\pgfmathresult 
	\pgfsetlinewidth{\yafaxiswidth} 
	\pgfsetcolor{yafaxiscolor}
	\pgfpathmoveto{\pgfpointadd{\pgfpointadd{\pgfplotspointrelaxisxy{0}{0}}{\pgfqpointxy{\xmincoord}{0}}}{\pgfqpoint{-0.5\yafaxiswidth}{\yafaxispad}}}
	\pgfpathlineto{\pgfpointadd{\pgfpointadd{\pgfplotspointrelaxisxy{0}{0}}{\pgfqpointxy{\xmaxcoord}{0}}}{\pgfqpoint{0.5\yafaxiswidth}{\yafaxispad}}}
	\pgfpathmoveto{\pgfpointadd{\pgfpointadd{\pgfplotspointrelaxisxy{0}{0}}{\pgfqpointxy{0}{\ymincoord}}}{\pgfqpoint{\yafaxispad}{-0.5\yafaxiswidth}}}
	\pgfpathlineto{\pgfpointadd{\pgfpointadd{\pgfplotspointrelaxisxy{0}{0}}{\pgfqpointxy{0}{\ymaxcoord}}}{\pgfqpoint{\yafaxispad}{0.5\yafaxiswidth}}}
	\pgfusepath{stroke}
}

\newcommand{\yafdrawYaxis}[2]{
	\pgfplotstransformcoordinatey{#1}\let\ymincoord=\pgfmathresult 
	\pgfplotstransformcoordinatey{#2}\let\ymaxcoord=\pgfmathresult 
	\pgfsetlinewidth{\yafaxiswidth} 
	\pgfsetcolor{yafaxiscolor}
	\pgfpathmoveto{\pgfpointadd{\pgfpointadd{\pgfplotspointrelaxisxy{0}{0}}{\pgfqpointxy{0}{\ymincoord}}}{\pgfqpoint{\yafaxispad}{-0.5\yafaxiswidth}}}
	\pgfpathlineto{\pgfpointadd{\pgfpointadd{\pgfplotspointrelaxisxy{0}{0}}{\pgfqpointxy{0}{\ymaxcoord}}}{\pgfqpoint{\yafaxispad}{0.5\yafaxiswidth}}}
	\pgfusepath{stroke}
}

\newcommand{\yafdrawaxisLimits}[4]{
	\pgfplotstransformcoordinatex{#1}\let\xmincoord=\pgfmathresult
	\pgfplotstransformcoordinatex{#2}\let\xmaxcoord=\pgfmathresult
	\pgfplotstransformcoordinatey{#3}\let\ymincoord=\pgfmathresult 
	\pgfplotstransformcoordinatey{#4}\let\ymaxcoord=\pgfmathresult 
	\pgfsetlinewidth{\yafaxiswidth} 
	\pgfsetcolor{yafaxiscolor}
	\pgfpathmoveto{
		\pgfpointadd{
			\pgfpointadd{
				\pgfplotspointrelaxisxy{0}{0}}{
				\pgfqpointxy{\xmincoord}{0}
			}
		}{
			\pgfqpoint{
				-0.5\yafaxiswidth}{
				\yafaxispad
			}
		}
	}
	\pgfpathlineto{
		\pgfpointadd{
			\pgfpointadd{
				\pgfplotspointrelaxisxy{0}{0}
			}{
				\pgfqpointxy{
					\xmaxcoord
				}{0}
			}
		}{
			\pgfqpoint{
				25.5\yafaxiswidth
			}{
				\yafaxispad
			}
		}
	}
	\pgfpathmoveto{\pgfpointadd{\pgfpointadd{\pgfplotspointrelaxisxy{0}{0}}{\pgfqpointxy{0}{\ymincoord}}}{\pgfqpoint{\yafaxispad}{-0.5\yafaxiswidth}}}
	\pgfpathlineto{\pgfpointadd{\pgfpointadd{\pgfplotspointrelaxisxy{0}{0}}{\pgfqpointxy{0}{\ymaxcoord}}}{\pgfqpoint{\yafaxispad}{0.5\yafaxiswidth}}}
	\pgfusepath{stroke}
}

\pgfplotscreateplotcyclelist{yaf}{% 
{yafcolor1,mark options={scale=0.75},mark=o}, 
{yafcolor2,mark options={scale=0.75},mark=square},
{yafcolor3,mark options={scale=0.75},mark=triangle},
{yafcolor4,mark options={scale=0.75},mark=o},
{yafcolor5,mark options={scale=0.75},mark=o},
{yafcolor6,mark options={scale=0.75},mark=o},
{yafcolor7,mark options={scale=0.75},mark=o},
{yafcolor8,mark options={scale=0.75},mark=o}} 

\pgfplotscreateplotcyclelist{yaf fill}{% 
{fill=yafcolor1,draw=none}, 
{fill=yafcolor2,draw=none},
{fill=yafcolor3,draw=none},
{fill=yafcolor4,draw=none},
{fill=yafcolor5,draw=none},
{fill=yafcolor6,draw=none},
{fill=yafcolor7,draw=none},
{fill=yafcolor8,draw=none}} 

\pgfplotsset{jv ybar/.style={
   ybar, 
   cycle list name=yaf fill,
   xtick = \empty,
   every extra x tick/.style={major tick length=0pt,color=black}, 
   xmajorgrids = false,
}}

\pgfplotsset{jv line/.style={
   no markers,
   cycle list name=yaf,
   log ticks with fixed point,
   y tick label style = {/pgf/number format/set thousands separator = {\,}},
}}
\pgfplotsset{jv line ylog/.style={
   jv line,
}}

\pgfplotsset{axis y line=left, axis x line=bottom,
	tick align=outside,
	tickwidth=\yafticklen,
	clip = false,
    x axis line style= {-, line width = 0pt, color=black!0},
    y axis line style= {-, line width = 0pt, color=black!0},
    x tick style= {line width = \yafaxiswidth, color=yafaxiscolor, yshift = \yafaxispad},
    y tick style= {line width = \yafaxiswidth, color=yafaxiscolor, xshift = \yafaxispad},
    x tick label style = {font=\scriptsize, yshift = \yaftlpad},
    y tick label style = {font=\scriptsize, xshift = \yaftlpad},
    every axis y label/.style = {at = {(ticklabel cs:0.5)}, rotate=90, anchor=center, font=\scriptsize, yshift = -\yaflabelpad},
    every axis x label/.style = {at = {(ticklabel cs:0.5)}, anchor=center, font=\scriptsize, yshift = \yaflabelpad},
    x tick label style = {font=\scriptsize, yshift = 1pt},
    grid = major,
    major grid style  = {dash pattern = on 1pt off 3 pt},
	every axis plot post/.append style= {line width=\yafaxiswidth} ,
	legend cell align = left,
	legend style = {inner sep = 1pt, cells = {font=\scriptsize}},
	legend image code/.code={%
		\draw[mark repeat=2,mark phase=2,#1] 
		plot coordinates { (0cm,0cm) (0.15cm,0cm) (0.3cm,0cm) };% 
	} 
}

\newcommand{\ourmaintitle}{Flexibly Mining Better Subgroups}
\newcommand{\ourtitle}{\ourmaintitle}
\newcommand{\oururl}{\url{http://eda.mmci.uni-saarland.de/flexi/}}

\newcommand{\codeurl}{\oururl}

\begin{document}

\title{\ourtitle}

\author{
Hoang-Vu Nguyen\thanks{Max Planck Institute for Informatics and Saarland University, Germany. Email: \texttt{\{hnguyen,jilles\}@mpi-inf.mpg.de}} \hspace{2.0cm}
Jilles Vreeken\footnotemark[1]
}
%\hspace*{1.25em}\texttt{\{skaraev,pmiettin,jilles\}@mpi-inf.mpg.de}}}

\date{}

\maketitle

\begin{abstract}
\small\baselineskip=9pt%
In subgroup discovery, also known as supervised pattern mining, discovering high quality one-dimensional subgroups and refinements of these is a crucial task. For nominal attributes, this is relatively straightforward, as we can consider individual attribute values as binary features. For numerical attributes, the task is more challenging as individual numeric values are not reliable statistics. Instead, we can consider combinations of adjacent values, i.e. bins. Existing binning strategies, however, are not tailored for subgroup discovery. That is, they do not directly optimize for the quality of subgroups, therewith potentially degrading the mining result. 

To address this issue, we propose \ourmethod. In short, with \ourmethod we propose to use optimal binning to find high quality binary features for \textit{both} numeric and ordinal attributes. We instantiate \ourmethod with various quality measures and show how to achieve efficiency accordingly. Experiments on both synthetic and real-world data sets show that \ourmethod outperforms state of the art with up to 25 times improvement in subgroup quality.

\end{abstract}

\section{Introduction} \label{sec:intro}

Subgroup discovery aims at finding subsets of the data, called subgroups, with high statistical unusualness with respect to the distribution of target variable(s)~\cite{wrobel:subgroup,kloesgen:subgroup,henrik:subgroup2}. It has applications in many areas, e.g.\ spatial analysis~\cite{kloesgen:subgroup}, marketing campaign management~\cite{lavrac:subgroup}, and health care~\cite{wouter:subgroup3}.

%In principle, each subgroup is described by a conjunction of conditions imposed on some attributes of the data. To search for subgroups, one needs a quality measure and a search scheme. The quality measure essentially captures the unusualness of subgroups and their generality (i.e.\ support).
A crucial part of the subgroup discovery process is the extraction of high quality binary features out of existing attributes. By binary features, we mean features whose values are either true or false. For instance, possible binary features of \textit{Age} attribute are \textit{Age} $\geq 50$ and $20 \leq$ \textit{Age} $\leq 30$. These features constitute one-dimensional subgroups or one-dimensional refinements of subgroups, which are used by many existing search schemes (e.g.\ beam search)~\cite{henrik:subgroup1,leeuwen:subgroup1,wouter:subgroup1}.

Deriving such features is straightforward for \textit{nominal} attributes, e.g.\ their individual values can be used directly as binary features~\cite{mampaey:subgroup}. This also is the case for \textit{ordinal} attributes if one is to treat them as nominal; the downside is that their ordinal nature is not used. The task, however, becomes more challenging for \textit{numerical} (e.g.\ real-valued) attributes. For such an attribute, binary features formed by single values statistically and empirically are not reliable; they tend to have low generality. Thus, one usually switches to combinations of adjacent values, i.e.\ \textit{bins}.

To this end, we observe three challenges that are in the way of finding high quality bins, i.e.\ binary features, for subgroup discovery. First, we need a problem formulation tailored to this purpose. Commonly used binning strategies such as equal-width and equal-frequency are oblivious of subgroup quality, impacting quality of the final output. Second, we should not place any restriction on the target; be it univariate or multivariate; nominal, ordinal, or numeric. Existing solutions also do not address this issue. For instance, \smdl~\cite{fayyad:discr} used in~\cite{henrik:subgroup2} requires that the target is univariate and nominal. Likewise, \RocInt~\cite{mampaey:subgroup} requires a univariate target. Third, the solution should scale well in order to handle large data sets. This means that we need new methods that can handle the first two issues and are efficient.

%The binning strategies employed in the literature so far do not address all challenges. In particular, some of them (e.g.\ equal-width and equal-frequency) are oblivious of subgroup quality, which greatly affects the final output quality. More advanced techniques in turn impose restrictions on the number and/or type of the target. For instance, the MDL-based binning~\cite{fayyad:discr} used in~\cite{henrik:subgroup2} requires that the target is univariate and nominal. Likewise, the method proposed in~\cite{mampaey:subgroup} requires a univariate target. Hence, it is still open how to achieve efficiency while addressing the first two challenges.

In this paper, we aim at tackling these challenges. We do so by proposing \ourmethod, for flexible subgroup discovery. In short, \ourmethod formulates the search of binary features per numeric/ordinal attribute as identifying the features with \textit{maximal average quality}. This formulation meets the generality requirement since it does not make any assumption on the target. We instantiate \ourmethod with various quality measures and show how to achieve efficiency accordingly. Extensive experiments on large real-world data sets show that \ourmethod outperforms state of the art, providing up to 25 times improvement in terms of subgroup quality. Furthermore, \ourmethod scales very well on large data sets.

The road map of this paper is as follows. In Section~\ref{sec:pre}, we present preliminaries. In Section~\ref{sec:one}, we introduce \ourmethod. In Sections~\ref{sec:div} and~\ref{sec:scale}, we plug different quality measures into our method and explain how to achieve efficiency. In Section~\ref{sec:rl}, we review related work. We present the experimental results in Section~\ref{sec:exp}. In Section~\ref{sec:dis} we round up with a discussion and conclude the paper in Section~\ref{sec:con}. For readability, we put all proofs in the appendix.

\section{Preliminaries} \label{sec:pre}

Let us consider a data set $\D$ of size $\size$ with attributes $\Ab = \{\A_1, \ldots, \A_\dima\}$, and targets $\Tb = \{\T_1, \ldots, \T_\dimt\}$. Each attribute $A \in \Ab$ can be nominal, ordinal, or numeric. When $A$ is either nominal or ordinal, its domain $\dom(A)$ is the set of its possible values. Each target $T \in \Tb$ can be either numeric or ordinal. If $T_i \in \Tb$ is numeric, we assume that $\dom(T_i) = [\minv_i, \maxv_i]$. Otherwise, $\dom(T_i)$ is the set of possible values of $T_i$. The probability function of $\Tb$ on $\D$ is denoted as $p(\Tb)$.

A subgroup $\Sub$ on $\D$ has the form $b_1 \wedge \ldots \wedge b_k$ ($k \in [1, \dima]$) where (1) each $b_j$ ($j \in [1, k]$) is a condition imposed on some attribute $A \in \Ab$ and (2) no two conditions share the same attribute. For each numeric attribute $\A$, each of its conditions $b$ has the form $A \in (l, u]$ where $l \in \mathbb{R} \cup \{-\infty\}$, $u \in \mathbb{R} \cup \{+\infty\}$, and $l < u$. If $A$ is ordinal, $b$ also has the form $A \in (l, u]$ where $l, u \in \dom(A)$ and $l < u$. If $A$ is categoric, $b$ instead has the form $A = a$ where $a \in \dom(A)$.

We let $\Subsett$ be the set of all subgroups on $\D$. The subset of $\D$ covered by $\Sub$ is denoted as $\D_{\Sub}$. We write $p_{\Sub}(\Tb)$ as the probability function of $\Tb$ on $\D_\Sub$. Overall, subgroup discovery is concerned with detecting $\Sub$ having high exception in its target distribution. The level of exception can be expressed through the divergence between $p_{\Sub}(\Tb)$ and $p(\Tb)$. To achieve high generality -- besides the divergence score -- the support $s = |\D_\Sub|$ of $\Sub$ should not be too small.

To quantify quality of subgroups, we need quality measure $\phi: \Subsett \rightarrow \mathbb{R}$ which assigns a score to each subgroup; the higher the score the better. Typically, $\phi$ needs to capture both unusualness of target distribution and subgroup support. In this paper, we will study five such quality measures.

%In the next section, we present our \ourmethod method for forming 1-D subgroups and 1-D refinements. Then, we show how to make it efficient for different quality measures.
%Our solution to both issues is applicable to \textit{any} search scheme that can handle numeric multivariate targets.

\section{Mining Binary Features} \label{sec:one}
%Formation of Initial 1-D Subgroups and 1-D Refinements for Numeric and Ordinal Attributes

\ourmethod mines binary features for attribute $\A$ that is either numeric or ordinal. When the features serve as one-dimensional subgroups on the first level of the search lattice, the entire realizations of $\A$ are used. For one-dimensional refinements, only those realizations covered by the subgroup in consideration are used~\cite{leeuwen:subgroup1,wouter:subgroup3}. For readability, we keep our discussion to the first case. The presentation can straightforwardly be adapted to the second case by switching from the context of the entire data set $\D$ to its subset covered by the subgroup to be refined. Below we also use \textit{bins} and \textit{binary features} interchangeably.
%We also use numeric targets as a representative; our discussion again can be easily adapted to ordinal targets.
%Na\"ive binning methods, e.g.\ equal-width or equal-frequency, unfortunately are oblivious of the distribution of targets. A more advanced binning algorithm~\cite{fayyad:discr} has been used in~\cite{henrik:subgroup2}; however, it is only applicable when $\D$ has a single nominal/discrete target.

%To overcome the drawbacks of existing work and to improve the quality of \sd on the numeric domain
In a nutshell, \ourmethod aims at finding binary features with maximal average quality. More specifically, it searches for the binning $\dsc$ of $\A$ such that the average quality of the bins formed by $\dsc$ is maximal. Formally, let $\Dsc$ be the set of possible binnings on $A$. For each $g \in \Dsc$, we let $\{b_g^1, \ldots, b_g^{|g|}\}$ be the set of bins formed by $g$ where $|g|$ is its number of bins. Each bin $b_g^i = (l_g^i, u_g^i]$ where $l_g^1 = -\infty$, $u_g^{|g|} = +\infty$, and $l_g^i = u_g^{i-1}$ for $i \in [2, |g|]$. \ourmethod solves for
$$\textstyle\dsc = \arg\max\limits_{g \in \Dsc} \frac{1}{|g|} \sum\limits_{i=1}^{|g|} \phi(b_g^i).$$
%In other words, we find the binning maximizing the average quality of 1-D subgroups formed on $\A$.
Another alternative would be to consider the \textit{sum} of subgroup quality. We discuss this option shortly afterward. Now, we present \ourmethod, our solution to the above problem.

At first, we note that $|\Dsc| = O(2^\size)$, i.e.\ the search space is exponential in $\size$ making an exhaustive enumeration infeasible. Fortunately, it is structured. In particular, for each $\lambda \in [1, \size]$ let $\dsc_\lambda$ be the optimal solution over all binnings producing $\lambda$ bins on $\A$. Let $\{b_\dsc^1, \ldots, b_\dsc^{\lambda}\}$ be its bins. We observe that for a fixed value of $\lambda$,
\begin{equation} \label{eq:recursive}
\textstyle\sum\limits_{i=1}^{\lambda} \phi(b_\dsc^i) = \phi(b_\dsc^\lambda) + \sum\limits_{i=1}^{\lambda - 1} \phi(b_\dsc^i)
\end{equation}
must be maximal. On the other hand, as $\dsc_\lambda$ is optimal w.r.t.\ $\lambda$, $\{b_\dsc^1, \ldots, b_\dsc^{\lambda - 1}\}$ must be the optimal way to partition values $\A \leq l_\dsc^{\lambda}$ into $\lambda - 1$ bins. Otherwise, we could have chosen a better way to do so. This consequently would produce another binning for all values of $\A$ such that (1) this binning has $\lambda$ bins and (2) it has a total quality higher than that of $\dsc_\lambda$. The existence of such a binning contradicts our assumption on $\dsc_\lambda$.

Hence, for each $\lambda$ its optimal binning $\dsc_\lambda$ exhibits optimal substructure. This motivates us to build a \textit{dynamic programming} algorithm to solve our problem.

%\vspace{0.5em}
\noindent\textbf{Algorithmic approach.}\ Our \ourmethod solution is in Algorithm~\ref{algo:two}. In short, it first forms bins $\{c_1, \ldots, c_\beta\}$ where $\beta \ll \size$. Each value $\qual[\lambda][i]$ where $\lambda \in [1, \beta]$ and $i \in [\lambda, \beta]$ stands for the total quality of bins obtained by optimally merging (discretizing) initial bins $c_1, \ldots, c_i$ into $\lambda$ bins. $b[\lambda][i]$ contains the resulting bins. Our goal is to efficiently compute $\qual[1\ldots\beta][\beta]$ and $b[1\ldots\beta][\beta]$. To do so, from Lines~4 to~6 we first compute $\qual[1][1\ldots\beta]$ and $b[1][1\ldots\beta]$. Then from Lines~7 to~14, we incrementally compute relevant elements of arrays $\qual$ and $b$, using the recursive relation described in Equation~(\ref{eq:recursive}). This is standard dynamic programming. Finally, we return the optimal binning after normalizing by the number of bins (Lines~15 and~16). There are two important points to note here.
%Then, for each $\lambda \in [1, \beta]$ it finds the optimal merging (binning) of $c_1, \ldots, c_i$ into $\lambda$ bins where $i \in [\lambda, \beta]$.

First, we form initial bins $\{c_1, \ldots, c_\beta\}$ of $\A$. Ideally, one would start with $O(\size)$ bins. However, the quality score $\phi(c)$ of bin $c$ is not reliable as well as not meaningful when its support $|c| = O(1)$. Thus, by pre-partitioning $\A$ in to $\beta$ bins, we ensure that there is sufficient data in each bin for a statistically reliable assessment of divergence. Choosing a suitable value for $\beta$ represents a tradeoff between accuracy and efficiency. We empirically study its effect in Section~\ref{sec:exp}. 
%Following~\cite{nguyen:ipd} we set $\beta = \sqrt{\size}$. 

Second, to ensure efficiency we need an efficient strategy to pre-compute $\phi(\bigcup_{k=j}^i c_k)$ (used in Lines~5, 9, and~10) for all $1 \leq j \leq i \leq \beta$. In the next section, we explain how to do this for different quality measures and analyze the complexity of \ourmethod accordingly.

%With a similar reasoning as above, we deduct that $\{b_\dsc^1, \ldots, b_\dsc^{\lambda-1}\}$ must be the optimal binning over all binnings producing $\lambda - 1$ bins for values $A \leq l_\dsc^{\lambda}$. Hence, the optimal binning $\dsc$ for each fixed value of $\lambda$ also exhibits optimal substructure, allowing us to build a \textit{dynamic programming} algorithm to solve the second problem. Our solution is given in Algorithm~\ref{algo:two}. 

%With the pre-computation as in our solution to the first problem the complexity of Algorithm~\ref{algo:two} is $O(\size^2 \dimt + \beta^3)$. By setting $\beta = \sqrt{\size}$ the complexity becomes $O(\size^2 \dimt)$. Thus, asymptotically Algorithms~\ref{algo:one} and~\ref{algo:two} have the same complexity. Our approach to boost up Algorithm~\ref{algo:one} in Section~\ref{sec:scale} also is applicable to Algorithm~\ref{algo:two}.

\begin{algorithm}[t]
\caption{\textsc{\ourmethod}}
\label{algo:two}
\begin{algorithmic}[1]
\STATE Create initial disjoint bins $\{c_1, \ldots, c_\beta\}$ of $\A$

\STATE Create a double array $\qual[1\ldots\beta][1\ldots\beta]$

\STATE Create an array $b[1\ldots\beta][1\ldots\beta]$ to store bins

\FOR{$i = 1 \rightarrow \beta$}
	\STATE $b[1][i] = \bigcup_{k=1}^{i} c_k$ and $\qual[1][i] = \phi(b[1][i])$
\ENDFOR

\FOR{$\lambda = 2 \rightarrow \beta$}
	\FOR{$i = \lambda \rightarrow \beta$}
		\STATE $\pos = \arg\max\limits_{1 \leq j \leq i-1} \qual[\lambda-1][j] + \phi(\bigcup_{k=j+1}^{i} c_k)$
		
		\STATE $\qual[\lambda][i] = \qual[\lambda-1][\pos] + \phi(\bigcup_{k=\pos+1}^{i} c_k)$
		
		\STATE Copy all bins in $b[\lambda-1][\pos]$ to $b[\lambda][i]$
		
		\STATE Add $\bigcup_{k=\pos+1}^{i} c_k$ to $b[\lambda][i]$
	\ENDFOR
\ENDFOR

\STATE $\lambda^{*} = \arg\max\limits_{1 \leq \lambda \leq \beta} \frac{1}{\lambda} \qual[\lambda][\beta]$

\STATE Return $b[\lambda^{*}][\beta]$
\end{algorithmic}
\end{algorithm}

%\vspace{0.5em}
\noindent\textbf{Alternative setting.}\ An intuitive alternate formulation of the problem is to maximize the \textit{total} quality of 1-D subgroups formed on $\A$. Formally, we have
$\dsc = \arg\max\limits_{g \in \Dsc} \sum\limits_{i=1}^{|g|} \phi(b_g^i)$, 
which can also be solved by dynamic programming (see Appendix~\ref{sec:alternative} for details). We compare to this setting in the experiments. We find that our standard setting, maximizing the average score, leads to much better results.

%The two problem settings are slightly different. In particular, we consider the total subgroup quality in the former and the average subgroup quality in the latter. As we will show, they need different solutions. At the first sight, the second problem seems to be more natural. However, we argue that the first one also makes sense. Our point is originated from our observation that the 1-D subgroups of $A$ at this stage are disjoint, i.e.\ partitional. In partitional clustering, a popular evaluation metric is the sum of squared errors (SSE)~\cite{jain:clustering} representing the total quality of the output clusters. Note that SSE is not normalized to the number of clusters when comparing clustering results with different numbers of clusters~\cite{sugar:clustering}. In our context, each binning is equivalent to a way to partition the data on $A$; $\phi(b_g^i)$ represents the quality of partition $b_g^i$. Thus, similarly to partitional clustering we can consider the total quality of all bins in lieu of their average quality. Further, intuitively we can see that the total quality is not biased by the number of bins. When there are many bins, the divergence score of each bin may be larger but their support is smaller. When there are a few bins, the divergence scores may get smaller but the supports become larger. The optimal solution of the first problem hence will strike for a good balance between divergence and support.

\section{Quality Measures} \label{sec:div}

\ourmethod works with any quality measure. In this section we show how to achieve efficiency, i.e.\ efficiently pre-compute $\phi(\bigcup_{k=j}^i c_k)$ for all $1 \leq j \leq i \leq \beta$, with various measures handling different types of targets. More specifically, we look at five measures: $\WRAcc$~\cite{henrik:subgroup2,henrik:subgroup3,leeuwen:subgroup1}, $\zscore$~\cite{mampaey:subgroup}, a measure based on Kullback-Leibler divergence ($\kl$)~\cite{leeuwen:subgroup1,leeuwen:subgroup2}, a measure based on Hellinger distance ($\hd$)~\cite{leman:emm}, and a measure based on quadratic measure of divergence ($\qr$)~\cite{nguyen:ipd}. We show characteristics of all measures in Table~\ref{tab:measures} and provide their details below. To simplify our analysis, we assume that each bin $c_i$ ($i \in [1, \beta]$) contains $\frac{\size}{\beta}$ objects.

\begin{table*}[t]
\centering 
\begin{tabular}{lccccccc}
\toprule

& \multicolumn{3}{l}{\textbf{Univariate}} &  & \multicolumn{3}{l}{\textbf{Multivariate}}\\

\cmidrule{2-4} \cmidrule{6-8}

{\bf Measure} & {\bf Nominal} & {\bf Ordinal} & {\bf Numeric} &  & {\bf Nominal} & {\bf Ordinal} & {\bf Numeric}\\[0.5em]

\otoprule

$\WRAcc$ & \cmark & \cmark & \xmark &  & \xmark & \xmark & \xmark\\

$\zscore$ & \xmark & \xmark & \cmark &  & \xmark & \xmark & \xmark\\

$\kl$ & \cmark & \cmark & \xmark &  & \cmark & \cmark & \xmark\\

$\hd$ & \cmark & \cmark & \xmark &  & \cmark & \cmark & \xmark\\

$\qr$ & \xmark & \cmark & \cmark &  & \xmark & \cmark & \cmark\\

\bottomrule
\end{tabular}
\caption{Characteristics of quality measures considered in this paper.} \label{tab:measures} 
\end{table*}

\subsection{$\WRAcc$ measure}

In subgroup discovery this measure is suited when $\D$ has a single binary target $\T$. That is, $\T$ assumes either a positive or a negative nominal value. Let $\size_{+}$ be the number of objects in $\D$ having positive target, i.e.\ positive label. Consider a subgroup $\Sub$ having $s = |\D_\Sub|$ objects; $s_{+}$ of which have positive label. The $\WRAcc$ score of $\Sub$ is defined as
$$\textstyle\WRAcc(\Sub) = \frac{s}{\size} (\frac{s_{+}}{s} - \frac{\size_{+}}{\size}).$$ 
Algorithm~\ref{algo:wracc} shows how to pre-compute $\WRAcc(\bigcup_{k=j}^i c_k)$ for all $1 \leq j \leq i \leq \beta$. The first for loop (Lines~2 to~5) is to count the number of positively labeled objects of $c_i$ ($i \in [1, \beta]$) and hence compute its $\WRAcc$ score. This step takes $O(\size)$. The nested loop (Lines~6 to~12) is to incrementally count the number of positively labeled objects of $\bigcup_{k=j}^i c_k$ and hence compute its $\WRAcc$ score. This step takes $O(\beta^2)$. Thus, Algorithm~\ref{algo:wracc} takes $O(\size + \beta^2)$.

Hence, \ourmethod with $\WRAcc$ measure (\wflexi) takes $O(\size + \beta^2 + \beta^3) = O(\size + \beta^3)$.

\begin{algorithm}[t]
\caption{\textsc{Pre-computation with $\WRAcc$}}
\label{algo:wracc}
\begin{algorithmic}[1]

\STATE Create an integer array $\countPos[1\ldots\beta]$

%\STATE Create a double array $\score[1\ldots\beta][1\ldots\beta]$

\FOR{$i = 1 \rightarrow \beta$}
	\STATE $\countPos[i] =$ \# of objects in $\D_{c_i}$ with positive label
	
	\STATE Compute $\WRAcc(c_i)$ based on $\countPos[i]$
	%\STATE $\score[i][i] = \WRAcc(c_i)$
\ENDFOR

\FOR{$i = 2 \rightarrow \beta$}
	\STATE $\theta = \countPos[i]$
	
	\FOR{$j = i - 1 \rightarrow 1$}
		\STATE $\theta = \theta + \countPos[j]$
		
		\STATE Set \# of objects with positive label in $\bigcup_{k=j}^i c_k$ to $\theta$ and hence compute $\WRAcc(\bigcup_{k=j}^i c_k)$
		
		%\STATE $\score[j][i] = \WRAcc(\bigcup_{k=j}^i c_k)$
	\ENDFOR
\ENDFOR
%\STATE Return $\score$
\end{algorithmic}
\end{algorithm}

\subsection{$\zscore$ measure}

This measure is suited when $\D$ has a single numeric target $\T$. Let $\mu_0$ and $\sigma_0$ be the mean and standard deviation of $\T$ in $\D$. Consider a subgroup $\Sub$ and let $\mu$ and $\sigma$ be the mean and standard deviation of $\T$ in $\Sub$. The quality of $\Sub$ w.r.t.\ $\zscore$ is defined as
$$\textstyle\zscore(\Sub) = \frac{\sqrt{s}}{\sigma_0} (\mu - \mu_0)$$
where $s = |\D_\Sub|$. To pre-compute $\zscore(\bigcup_{k=j}^i c_k)$ for all $1 \leq j \leq i \leq \beta$, we can re-use Algorithm~\ref{algo:wracc} with a few modifications. The new algorithm is in Algorithm~\ref{algo:zscore}. It also takes $O(\size + \beta^2)$.

Hence, \ourmethod with $\zscore$ measure (\zflexi) has the same complexity as \wflexi.

\begin{algorithm}[t]
\caption{\textsc{Pre-computation with $\zscore$}}
\label{algo:zscore}
\begin{algorithmic}[1]

\STATE Create an integer array $\binMean[1\ldots\beta]$

%\STATE Create a double array $\score[1\ldots\beta][1\ldots\beta]$

\FOR{$i = 1 \rightarrow \beta$}
	\STATE $\binMean[i] =$ target mean in $c_i$
	
	\STATE Compute $\zscore(c_i)$ based on $\binMean[i]$
	%\STATE $\score[i][i] = \zscore(c_i)$
\ENDFOR

\FOR{$i = 2 \rightarrow \beta$}
	\STATE $\theta = |c_i|$
	
	\STATE $\mu = \binMean[i]$
	
	\FOR{$j = i - 1 \rightarrow 1$}
		\STATE $\mu = \theta \times \mu + |c_j| \times \binMean[j]$
		
		\STATE $\theta = \theta + |c_j|$
		
		\STATE Set target mean in $\bigcup_{k=j}^i c_k$ to $\mu$ and hence compute $\zscore(\bigcup_{k=j}^i c_k)$
		%\STATE $\score[j][i] = \zscore(\bigcup_{k=j}^i c_k)$
	\ENDFOR
\ENDFOR
%\STATE Return $\score$
\end{algorithmic}
\end{algorithm}

\subsection{$\kl$ measure} \label{sec:klmeasure}

This measure is suited to $\D$ with univariate/multivariate nominal and/or ordinal target. W.l.o.g., assume that we have multivariate target $\Tb = \{\T_1, \ldots, \T_\dimt\}$. The $\kl$ score of each subgroup $\Sub$ is defined as
$$\textstyle\kl(\Sub) = \frac{s}{\size} \sum\limits_{t_1, \ldots, t_\dimt} p_\Sub(t_1, \ldots, t_\dimt) \times \log \frac{p_\Sub(t_1, \ldots, t_\dimt)}{p(t_1, \ldots, t_\dimt)}$$
where $s = |\D_\Sub|$. A straightforward computation of $\kl(\bigcup_{k=j}^i c_k)$ for every $1 \leq j \leq i \leq \beta$ is done by considering only $(t_1, \ldots, t_\dimt)$ that appears in the data covered by $\Sub$. This is because $p_\Sub(t_1, \ldots, t_\dimt) \times \log \frac{p_\Sub(t_1, \ldots, t_\dimt)}{p(t_1, \ldots, t_\dimt)} = 0$ for $(t_1, \ldots, t_\dimt)$ not in $\Sub$. As $p_\Sub(t_1, \ldots, t_\dimt)$ and $p(t_1, \ldots, t_\dimt)$ can be efficiently calculated using hash tables, computing $\kl(\bigcup_{k=j}^i c_k)$ takes $O((i - j + 1) \times \dimt \times \frac{\size}{\beta})$. The pre-computation hence in total takes
$$\textstyle\sum\limits_{i=1}^\beta \sum\limits_{j=1}^{i} O((i - j + 1) \times \dimt \times \frac{\size}{\beta}),$$
which can be simplified to $O(\size \beta^2 \dimt)$.

Thus, \ourmethod with $\kl$ (\kflexi) takes $O(\size \beta^2 \dimt + \beta^3) = O(\size \beta^2 \dimt)$ as $\beta \ll \size$.

\subsection{$\hd$ measure}

Similarly to $\kl$ measure, $\hd$ measure is suited to $\D$ with univariate/multivariate nominal and/or ordinal target. The $\hd$ score of a subgroup $\Sub$ is defined as
\begin{align*}
& \textstyle\hd(\Sub) = \left(- \frac{s}{\size} \log \frac{s}{\size} - \frac{\size - s}{\size} \log \frac{\size - s}{\size}\right) \\
& \textstyle\quad \quad \times \sum\limits_{t_1, \ldots, t_\dimt} \left(\sqrt{p_\Sub(t_1, \ldots, t_\dimt)} - \sqrt{p(t_1, \ldots, t_\dimt)}\right)^2.
\end{align*}
where $s = |\D_\Sub|$. The pre-computation is done similarly to Section~\ref{sec:klmeasure}. However, we here need to consider $(t_1, \ldots, t_\dimt)$ that appears in $\D$, not just in $\Sub$. Thus, for $(i, j)$ where $j \leq i$, computing $\hd(\bigcup_{k=j}^i c_k)$ takes $O(\size \dimt)$. Hence, the cost of the pre-computation is identical to that of \kflexi.

In other words, \ourmethod with $\hd$ measure (\hflexi) has the same complexity as \kflexi.

\subsection{$\qr$ measure}

To handle univariate/multivariate numeric and/or ordinal targets, we propose $\qr$ measure which is based on $\id$ -- a quadratic measure of divergence~\cite{nguyen:ipd}. We pick $\id$ as it is applicable to both univariate and multivariate data. In addition, its computation on empirical data is in closed form formula, i.e.\ it is highly suited to exploratory data analysis. Originally, $\id$ is used for numeric data. Our $\qr$ measure improves over this by adapting $\id$ to ordinal data. This enables $\qr$ to handle multivariate numeric targets, as well as multivariate targets whose types are a mixed of numeric and ordinal. As shown in Table~\ref{tab:measures}, no previous measure is able to achieve this. By making $\qr$ work with \ourmethod, we can further demonstrate the flexibility and generality of our solution. The details are as follows.

Consider a subgroup $\Sub$ with $s = |\D_\Sub|$ objects. W.l.o.g., assume that there are multiple targets. The $\qr$ score of $\Sub$ is
$$\textstyle\qr(\Sub) = f(s) \times \id(p_\Sub(\Tb)\; ||\; p(\Tb))$$
where $f(s)$ is either $\frac{s}{\size}$ (following~\cite{henrik:subgroup2,leeuwen:subgroup2}) or $\left(\frac{s}{\size} \log \frac{s}{\size} - \frac{\size - s}{\size} \log \frac{\size - s}{\size}\right)$ (following~\cite{leman:emm,wouter:subgroup1}). When all targets are numeric, we have $\id(p_\Sub(\Tb)\; ||\; p(\Tb)) = $
$$\textstyle\int_{\minv_1}^{\maxv_1} \cdots \int_{\minv_\dimt}^{\maxv_\dimt} \left(P_\Sub(t_1, \ldots, t_\dimt) - P(t_1, \ldots, t_\dimt)\right)^2 dt_1 \cdots dt_\dimt$$
where $P_\Sub(.)$ and $P(.)$ are the cdfs of $p_\Sub(.)$ and $p(.)$, respectively. We extend to ordinal targets by replacing $\int_{\minv_i}^{\maxv_i} dt_i$ with $\sum\limits_{t_i \in \dom(T_i)}$ for each ordinal $T_i$.

Similarly to $\id$, our $\qr$ measure also permits computation on empirical data in closed form. More specifically, let the empirical data of $\D$ be $\{\D^1, \ldots, \D^\size\}$. Similarly, let the empirical data of $\D_\Sub$ be $\{\D_\Sub^1, \ldots, \D_\Sub^s\}$ where $s = |\D_\Sub|$. We write $\D^1_i$ and $\D_{\Sub,i}^1$ as the projections of $\D^1$ and respectively $\D_\Sub^1$ on $\T_i$. We have the following.

\begin{theorem} \label{theo:quadem}
Empirically, $\qr(p_\Sub(\Tb)\; ||\; p(\Tb)) =$
\begin{align*}
&\textstyle f(s) \times \left(\frac{1}{s^2} \sum\limits_{i=1}^{s}\sum\limits_{j=1}^{s} \prod\limits_{k=1}^{\dimt} h_k(\D_{\Sub,i}^k, \D_{\Sub,j}^k) \right. \\
&\left. \textstyle\qquad \qquad - \frac{2}{s \size} \sum\limits_{i=1}^{s}\sum\limits_{j=1}^{\size} \prod\limits_{k=1}^{\dimt} h_k(\D_{\Sub,i}^k, \D_j^k) \right. \\ 
&\left. \textstyle\qquad \qquad + \frac{1}{\size^2} \sum\limits_{i=1}^{\size}\sum\limits_{j=1}^{\size} \prod\limits_{k=1}^{\dimt} h_k(\D_i^k, \D_j^k)\right)
\end{align*}
where $h_k\left(x, y\right) = \left(\maxv_k - \max(x, y)\right)$ if $T_k$ is numeric, and $h_k\left(x, y\right) = \sum\limits_{t \in \dom(T_k)} \textbf{I}\left(t \geq \max(x, y)\right)$ if $T_k$ is ordinal. Here, $\textbf{I}(.)$ is an indicator function.
\end{theorem}

\proofApx

Following Theorem~\ref{theo:quadem}, to obtain $\qr(p_\Sub(\Tb)\; ||\; p(\Tb))$ we need to compute three terms -- referred to as $\Sub.e_1$, $\Sub.e_2$, and $\Sub.e_3$ -- where
$$\textstyle\qr(p_\Sub(\Tb)\; ||\; p(\Tb)) = f(s) \times \left(\frac{1}{s^2} \Sub.e_1 - \frac{2}{s \size} \Sub.e_2 + \frac{1}{\size^2} \Sub.e_3\right).$$
Note that $e = \Sub.e_3$ is independent of $\Sub$ and thus needs to be computed only once for all subgroups. We now prove a property of $\qr$ which is important for efficiently pre-computing $\qr(\bigcup_{k=j}^i c_k)$ for all $1 \leq j \leq i \leq \beta$.

\begin{lemma} \label{lem:quadcom}
Let $\Sub$ and $\Sube$ be two consecutive non-overlapping bins of attribute $\A$, i.e.\ $\D_\Sub \cap \D_\Sube = \emptyset$. Let $Y = \Sub \cup \Sube$, $s = |\D_\Sub|$, and $r = |\D_\Sube|$. It holds that $Y.e_1 = \Sub.e_1 + \Sube.e_1 + 2 \inter(\Sub, \Sube)$ and $Y.e_2 = \Sub.e_2 + \Sube.e_2$ where $\inter(\Sub, \Sube) = \sum\limits_{i=1}^s \sum\limits_{j=1}^r \prod\limits_{k=1}^\dimt h_k(\D_{\Sub,i}^k, \D_{\Sube,j}^k)$.
\end{lemma}

\proofApx

Lemma~\ref{lem:quadcom} tells us that terms $e_1$ and $e_2$ of a bin made up by joining two adjacent non-overlapping bins $S$ and $R$ can be obtained from the terms of $S$ and $R$, and $\inter(\Sub, \Sube)$. Note that $\inter$ is symmetric. Further, we prove that it is additive -- a property that is also important for the pre-computation.

\begin{lemma} \label{lem:inter}
Let $\Sube_1, \ldots, \Sube_l$, and $\Sub$ be non-overlapping bins of $\A$ such that $\Sube_i$ is adjacent to $\Sube_{i+1}$ for $i \in [1, l-1]$, and $\Sube_l$ is adjacent to $\Sub$. It holds that
$$\textstyle\inter\left(\Sub, \bigcup_{i=1}^l \Sube_i\right) = \sum\limits_{i=1}^l \inter(\Sub, \Sube_i).$$
\end{lemma}

\proofApx

Algorithm~\ref{algo:qr} summarizes how to compute $\qr(\bigcup_{k=j}^i c_k)$ for all $1 \leq j \leq i \leq \beta$. The details are as follows.

\begin{itemize}
\item First, we compute terms $e_1$ and $e_2$, and $\qr(c_i)$ for each $i \in [1, \beta]$ (Line~1): This step takes $O(\size \times \frac{\size}{\beta} \times \dimt)$ for each $c_i$, i.e.\ its total cost is $O(\size^2 \dimt)$.

\item Second, we compute $\inter(c_j, c_i)$ for each $j \in [1, \beta-1]$ and $i \in [j+1, \beta]$ (Line~2): This step takes $O(\frac{\size^2}{\beta^2} \dimt)$ for each pair $(j, i)$, i.e.\ its total cost is $O(\size^2 \dimt)$.

\item Third, we compute $\inter(\bigcup_{k=j}^{i-1} c_k, c_i)$ for each $i \in [2, \beta]$ and $j \in [1, i-1]$ (Lines~3 to~9): We use the fact that $\inter(\bigcup_{k=j}^{i-1} c_k, c_i) = \sum_{k=j}^{i-1} \inter(c_k, c_i)$ (see Lemma~\ref{lem:inter}). This step takes $O(\beta^2)$.

\item Fourth, we compute terms $e_1$ and $e_2$, and $\qr(\bigcup_{k=j}^i c_k)$ for each $i \in [2, \beta]$ and $j \in [1, i-1]$ (Lines~10 to~14): From Lemma~\ref{lem:quadcom}, terms $e_1$ and $e_2$ of $\bigcup_{k=j}^i c_k$ can be computed based on the terms of $\bigcup_{k=j}^{i-1} c_k$, $c_i$, and $\inter(\bigcup_{k=j}^{i-1} c_k, c_i)$. This step takes $O(\beta^2)$.
\end{itemize}

Overall, Algorithm~\ref{algo:qr} takes $O(\size^2 \dimt)$. Thus, \ourmethod with $\qr$ (\qflexi) takes $O(\size^2 \dimt + \beta^3) = O(\size^2 \dimt)$ as $\beta \ll \size$.

\begin{algorithm}[t]
\caption{\textsc{Pre-computation with $\qr$}}
\label{algo:qr}
\begin{algorithmic}[1]

\STATE Compute terms $e_1$ and $e_2$, and $\qr(c_i)$ for $c_i$ ($i \in [1, \beta]$)

\STATE Compute $\inter(c_j, c_i)$ for every $j \in [1, \beta - 1]$ and $i \in [j+1, \beta]$

\FOR{$i = 2 \rightarrow \beta$}
	\STATE $\theta = 0$

	\FOR{$j = i - 1 \rightarrow 1$}
		\STATE $\theta = \theta + \inter(c_j, c_i)$
	
		\STATE Set $\inter(\bigcup_{k=j}^{i-1} c_k, c_i)$ to $\theta$
	\ENDFOR
\ENDFOR

\FOR{$i = 2 \rightarrow \beta$}
	\FOR{$j = 1 \rightarrow i - 1$}
		\STATE Compute terms $e_1$ and $e_2$, and $\qr(\bigcup_{k=j}^i c_k)$ for $\bigcup_{k=j}^i c_k$ using the terms of $\bigcup_{k=j}^{i-1} c_k$, $c_i$, and $\inter(\bigcup_{k=j}^{i-1} c_k, c_i)$
	\ENDFOR
\ENDFOR

\end{algorithmic}
\end{algorithm}

\subsection{Remarks}

As $\beta$ is typically small (from 5 to 40), \wflexi, \zflexi, \kflexi, and \hflexi scale linearly in $\size$. On the other hand, \qflexi scales quadratic in $\size$ regardless which value $\beta$ takes. In Section~\ref{sec:scale}, we propose a method to boost the efficiency of \qflexi.

%Besides, in some settings of subgroup discovery one may be interested in defining the quality of a subgroup $\Sub$ based on the divergence between $p_\Sub(\Tb)$ and $p_{\Subc}(\Tb)$~\cite{leman:emm}. Here $\Subc$ is the complement of $\Sub$, i.e.\ $\D_\Sub \cap \D_{\Subc} = \emptyset$ and $\D_\Sub \cup \D_{\Subc} = \D$. In Appendix~\ref{sec:ext}, we show how to adapt our methods to this setting.

\section{Improving Scalability} \label{sec:scale}

The complexity of \qflexi is quadratic in $\size$, which may become a disadvantage on large data. We thus propose a solution to alleviate the issue. Again, we keep our discussion to the case of one-dimensional subgroups. The case of refinements straightforwardly follows.

We observe that the performance bottleneck is the pre-computations of $\qr(c_i)$ ($i \in [1, \beta]$) and $\inter(c_j, c_i)$ ($j \in [1, \beta-1]$ and $i \in [j+1, \beta]$). In fact, keys to these quantities are the distributions $p_{c_i}(\Tb)$ in bins $c_i$ ($i \in [1, \beta])$. In our computation the data of $\D_{c_i}$ projected onto $\Tb$, denoted as $\D_{c_i,\Tb}$, is considered to be i.i.d.\ samples of the (unknown) pdf $p_{c_i}(\Tb)$. By definition, i.i.d.\ samples are obtained by randomly sampling from an infinite population or by randomly sampling with replacement from a finite population~\cite{thompson:sampling}. In both cases, the distribution of i.i.d.\ samples are assumed to be identical to the distribution of the population. This is especially true when the sample size is very large~\cite{scott:density}. Thus, when $\size$ is very large the size of $\D_{c_i,\Tb}$ -- which is $\frac{\size}{\beta}$ -- is also large. This makes the empirical distribution $\hat{p}_{c_i}(\Tb)$ formed by $\D_{c_i,\Tb}$ approach the true distribution $p_{c_i}(\Tb)$.

Assume now that we randomly draw with replacement $\epsilon \times \frac{\size}{\beta}$ samples $\da_{c_i,\Tb}$ from $\D_{c_i,\Tb}$ where $\epsilon \in (0, 1)$. As mentioned above, $\da_{c_i,\Tb}$ contains i.i.d.\ samples of $\hat{p}_{c_i}(\Tb) \approx p_{c_i}(\Tb)$. As with any set of i.i.d.\ samples with a reasonable size, we can assume that the distribution of $\Tb$ in $\da_{c_i,\Tb}$ is identical to $p_{c_i}(\Tb)$.
%When $\size$ is very large, $|\D_{c_i, \Tb}| = \sqrt{\size}$ is large as well. With a large amount of empirical data we can subsample from it to construct i.i.d.\ samples representing reasonably well the true distribution~\cite{kollios:biased,gu:sampling}. Hence, the larger $\size$ is the larger the chance that we can construct $\hat{\D}_{c_i,\Tb}$ out of $\D_{c_i,\Tb}$ featuring a good estimate of $p_{c_i}(\Tb)$.

Based on this line of reasoning, when $\size$ is large we propose to randomly subsample with replacement the data in each bin $c_i$ ($i \in [1, \beta]$) for our computation. The important point here is to identify how large $\epsilon$ should be, i.e.\ how many samples we should use. We will show that a low value of $\epsilon$ already suffices, e.g.\ $\epsilon = 0.1$. If we subsample the bins while not subsampling $\D$ (in the same way) for computing quality scores, the complexity of \qflexi is $O(\epsilon \size^2 \dimt)$. If we subsample $\D$ as well, its complexity is then $O(\epsilon^2 \size^2 \dimt)$.
%For each bin $c_i$, the number of such samples is $\frac{\size}{\beta} = \sqrt{\size}$.

\section{Related Work} \label{sec:rl}

Traditionally, subgroup discovery focuses on nominal attributes~\cite{kloesgen:subgroup,lavrac:subgroup,henrik:subgroup1,henrik:subgroup3}. More recent work~\cite{leeuwen:subgroup1,wouter:subgroup1,wouter:subgroup3,lemmerich:subgroup} considers numeric attributes, employing equal-width or equal-frequency binning to create binary features. These strategies however do not optimize quality of the features generated, which consequently affects the final output quality.

To alleviate this, Grosskreutz and R\"uping~\cite{henrik:subgroup2} employ \smdl~\cite{fayyad:discr}. It requires that the target is univariate and nominal. Further, it finds the bins optimizing the divergence between $p_b(\Tb)$ and $p_{b'}(\Tb)$ where $b$ and $b'$ are two arbitrary consecutive bins. That is, only local distributions of the target (within individual bins) are compared to each other. The goal of subgroup discovery in turn is to assess the divergence between $p_b(\Tb)$ and $p(\Tb)$~\cite{henrik:subgroup1}. While \smdl improves over na\"ive binning methods, it does not directly optimize subgroup quality.

Mampaey et al.~\cite{mampaey:subgroup} introduce \RocInt, which searches for the binary feature with highest quality on each numeric/ordinal attribute. It does so by analyzing the coverage space, reminiscent of ROC spaces, of the \textit{univariate} target. \RocInt and \ourmethod are different in many aspects. First, \RocInt is suitable for univariate targets only. \ourmethod in turn works with both univariate and multivariate targets. Second, as \RocInt finds the best feature per attribute, it is not for mining one-dimensional subgroups. In fact, it is designed for mining one-dimensional refinements. On the contrary, \ourmethod can find both types of pattern. Third, \RocInt requires $\phi$ to be convex. \ourmethod in turn works with any type of quality measures.

Besides the binning methods discussed above, there exist also other techniques applicable to -- albeit not yet studied in -- subgroup discovery. For instance, \unml~\cite{petri:mdl} mines bins per numeric attribute that best approximate its true distribution. On the other hand, multivariate binning techniques (e.g.\ \ipd~\cite{nguyen:ipd}) focus on optimizing the divergence between local distributions in individual bins. Overall, these methods do not optimize subgroup quality.

Regarding quality measure $\phi$, majority of existing ones focus on univariate targets~\cite{kloesgen:subgroup,lavrac:subgroup,henrik:subgroup1,henrik:subgroup2,henrik:subgroup3,leeuwen:subgroup3,mampaey:subgroup,wouter:subgroup3,lemmerich:subgroup}. Van Leeuwen and Knobbe~\cite{leeuwen:subgroup1,leeuwen:subgroup2} propose a measure based on Kullback-Leibler divergence for multivariate nominal/ordinal targets. Their measure is reminiscent of $\kl$ measure in Section~\ref{sec:klmeasure}; yet, they assume the targets are statistically independent while $\kl$ takes into account interaction of targets. Also for multivariate nominal/ordinal targets, Duivesteijn et al.~\cite{wouter:subgroup1} introduce a measure based on Bayesian network. Measures for multivariate numeric targets appear mainly in exceptional model mining (EMM)~\cite{leman:emm,wouter:subgroup4}. Consequently, such measures are model-based. Our $\qr$ measure in turn is purely non-parametric. A non-parametric measure for multivariate numeric targets is recently introduced in~\cite{nguyen:cjs}. Unlike this measure as well as measures of EMM, $\qr$ can handle multivariate targets whose types are a mixed of numeric and ordinal.

\section{Experiments} \label{sec:exp}

\begin{table}[tb]
\centering 
\begin{tabular}{lrrrrr}
\toprule

& & \multicolumn{4}{l}{{\bf Attributes}}\\
\cmidrule{3-6}

{\bf Data} & {\bf Rows} & {\bf Nom.} & {\bf Ord.} & {\bf Num.} & {\bf Total}\\
\otoprule

Adult & 48\,842 & 7 & 1 & 6 & 14\\

Bike & 17\,379 & 5 & 3 & 7 & 15\\

Cover & 581\,012 & 44 & 3 & 7 & 54\\

Gesture & 9\,900 & 1 & 0 & 32 & 33\\

Letter & 20\,000 & 1 & 0 & 16 & 16\\

Bank & 45\,211 & 11 & 2 & 8 & 21\\

Naval & 11\,934 & 0 & 0 & 18 & 18\\

Network & 53\,413 & 1 & 9 & 14 & 24\\

SatImage & 6\,435 & 1 & 0 & 36 & 37\\

Drive & 58\,509 & 1 & 0 & 48 & 49\\

Turkiye & 5\,820 & 0 & 32 & 1 & 33\\

Year & 515\,345 & 1 & 0 & 90 & 91\\

\bottomrule
\end{tabular}
\caption{Characteristics of real-world data sets. `Nom.', 'Ord.', and 'Num.' are abbreviations of respectively `Nominal', `Ordinal', and `Numerical'.} \label{tab:datasets} 
\end{table}

In this section, we empirically evaluate \ourmethod by plugging it into beam search -- a common search scheme of subgroup discovery~\cite{leman:emm,leeuwen:subgroup1,wouter:subgroup1}. We aim at examining if \ourmethod is able to efficiently and effectively discover subgroups of high quality. For a comprehensive assessment, we test with all five quality measures discussed above. As performance metric, we use the average quality of top 50 subgroups. We also study the parameter setting of \ourmethod; this includes the effect of our scalability improvement for $\qr$ measure (see Section~\ref{sec:scale}). We implemented \ourmethod in Java, and make our code available for research purposes.\!\footnote{\codeurl} All experiments were performed single-threaded on an Intel(R) Core(TM) i7-4600U CPU with 16GB RAM. We report wall-clock running times.

We compare \ourmethod to \alter which finds bins optimizing the sum of quality instead of average quality, \ef for equal-frequency binning, and \ew for equal-width binning. As further baselines, we test with state of the art \textit{supervised} discretization \smdl~\cite{fayyad:discr}, \textit{unsupervised univariate} discretization \unml~\cite{petri:mdl}, and \textit{unsupervised multivariate} discretization \ipd~\cite{nguyen:ipd}. For measures that handle univariate targets only ($\WRAcc$ and $\zscore$), we test with \unml and exclude \ipd. For the other three measures, we use \ipd instead. Finally, we include \RocInt~\cite{mampaey:subgroup}, state of the art method on mining binary features for subgroup discovery. For each competitor, whenever applicable we try with different parameter settings and pick the best result. For \ourmethod, by default we set the number of initial bins $\beta = 20$; and when subsampling is used, we set the subsampling rate $\epsilon = 0.1$. We form initial bins $\{c_1, \ldots, c_\beta\}$ by applying equal-frequency binning; this procedure has also been used in~\cite{ReshefEtAl2011,nguyen:ipd,nguyen:mac}.

We experiment  with 12 real-world data sets drawn from the UCI Machine Learning Repository. Their details are in Table~\ref{tab:datasets}. To show that \ourmethod methods are suited to subgroup discovery on large-scale data, 9 data sets we pick have more than 10\,000 records. For brevity, in the following we present the results on 6 data sets with largest sizes: Adult, Cover, Bank, Network, Drive, and Year. For conciseness, we keep our discussion to \wflexi, \kflexi, and \qflexi; we postpone further results to Appendix~\ref{sec:full}.

\subsection{Quality results with $\WRAcc$} \label{sec:wraccres}

As $\WRAcc$ requires univariate binary target, we follow~\cite{wouter:subgroup3} to convert nominal (but non-binary) targets to binary.
The results are in Tables~\ref{tab:wracc}. Here, we display the absolute as well as relative average quality (for other measures we show relative quality only). For the relative quality, the scores of \wflexi are the bases (100\%). Going over the results, we see that \wflexi gives the best average quality in all data sets. Its performance gain over the competitors is up to 300\%. Note that by optimizing average subgroup quality, instead of total quality as \alter does, \wflexi mines better binary features and hence achieves better performance than \alter. \ef, \ew, \smdl, and \unml form binary features oblivious of subgroup quality and perform less well. \RocInt, on the other hand, performs better, but as it forms one feature per attribute at each level of the search it makes the search more sensitive to local optima.
%All in all, the results show that \ourmethod is a promising solution towards mining binary features for subgroup discovery.
%The results with $\zscore$ are in Tables~\ref{tab:zscore}. \ourmethod attains the best average quality in 5 out of 6 data sets. Though its margins to competitors are now smaller, the performance gain it brought is high on Cover, Drive, and Year---3 data sets with largest sizes.
%$\zscore$ in turn requires univariate numeric target. Thus, for each data set we randomly select a numeric attribute as the target.

\begin{table*}[tb]
\centering 
\begin{tabular}{lrrrrrrr}
\toprule
{\bf Data} & {\bf \wflexi} & {\bf \alter} & {\bf \ef} & {\bf \ew} & {\bf \smdl} & {\bf \unml} & {\bf \RocInt}\\
\otoprule

Adult & \textbf{0.08 (100)} & 0.07 (88) & 0.07 (88) & 0.07 (88) & 0.07 (88) & 0.06 (75) & 0.07 (88)\\

%Bike & \textbf{0.0639} & 0.0449 & 0.0383 & 0.0455 & 0.0597 & 0.0383 & 0.0464\\

Cover & \textbf{0.12 (100)} & 0.11 (92) & 0.04 (33) & 0.08 (66) & 0.04 (33) & 0.05 (42) & 0.04 (33)\\

%Gesture & \textbf{0.0991} & 0.0857 & 0.0281 & 0.0937 & 0.0748 & 0.0351 & 0.0407\\

%Letter & \textbf{0.0791} & 0.0492 & 0.0225 & 0.0294 & 0.0479 & 0.0446 & 0.0389\\

Bank & \textbf{0.04 (100)} & 0.03 (75) & 0.02 (50) & 0.03 (75) & 0.02 (50) & 0.02 (50) & 0.02 (50)\\

%Naval &  &  &  &  &  &  & \\

Network & \textbf{0.18 (100)} & 0.13 (72) & 0.10 (56) & 0.12 (67) & 0.14 (78) & 0.12 (67) & 0.14 (78)\\

%SatImage & \textbf{0.1491} & 0.1103 & 0.0286 & 0.0481 & 0.0951 & 0.0414 & 0.0511\\

Drive & \textbf{0.11 (100)} & 0.08 (73) & 0.03 (27) & 0.08 (73) & 0.05 (45) & 0.06 (55) & 0.05 (45)\\

%Turkiye & \textbf{0.1105} & 0.1093 & 0.1058 & 0.1058 & 0.1059 & 0.1005 & 0.1076\\

Year & \textbf{0.12 (100)} & 0.08 (67) & 0.06 (50) & 0.06 (50) & 0.07 (58) & 0.06 (50) & 0.07 (58)\\

\bottomrule
\end{tabular}
\caption{[Higher is better] Average quality, measured by $\WRAcc$, of top 50 subgroups. We give both the absolute scores, as well as the relative results (in brackets) compared to \wflexi.} \label{tab:wracc} 
\end{table*}

\subsection{Quality results with $\kl$}

We recall that $\kl$ is suited to univariate/multivariate nominal and/or ordinal targets. For Adult and Bank, we use all nominal attributes as targets. For Cover, we randomly select 27 nominal attributes as targets. For Network, we combine nominal and ordinal attributes to create the targets. Drive and Year both have one nominal attribute and no ordinal one. Thus, for each of them we use the nominal attribute as univariate target.

The results are in Table~\ref{tab:kl}. \kflexi achieves the best performance in all data sets. It yields up to 25 times quality improvement compared to competing methods. Note that \smdl and \RocInt both require univariate targets and hence are not applicable to Adult, Cover, Bank, and Network. \kflexi in turn is suited to both univariate and multivariate targets.

\begin{table}[t]
\centering 
\begin{tabular}{lrrrrrrr}
\toprule
{\bf Data} & {\bf \kflexi} & {\bf \alter} & {\bf \ef} & {\bf \ew} & {\bf \smdl} & {\bf \ipd} & {\bf \RocInt}\\
\otoprule

Adult & \textbf{100} & 38 & 37 & 31 & \emph{n/a} & 4 & \emph{n/a}\\

%Bike & \textbf{0.50} & 0.26 & 0.34 & 0.35 & - & 0.05 & -\\

Cover & \textbf{100} & 43 & 64 & 75 & \emph{n/a} & 45 & \emph{n/a}\\

%Gesture & \textbf{0.53} & 0.22 & 0.33 & 0.33 & 0.50 & 0.31 & 0.33\\

%Letter & \textbf{0.52} & 0.43 & 0.43 & 0.47 & 0.43 & 0.06 & 0.43\\

Bank & \textbf{100} & 46 & 62 & 33 & \emph{n/a} & 6 & \emph{n/a}\\

Network & \textbf{100} & 55 & 68 & 55 & \emph{n/a} & 21 & \emph{n/a}\\

%SatImage & \textbf{0.53} & 0.28 & 0.37 & 0.48 & 0.45 & 0.26 & 0.37\\

Drive & \textbf{100} & 42 & 64 & 85 & 89 & 42 & 62\\

%Turkiye & \textbf{0.53} & 0.50 & 0.50 & 0.50 & - & 0.15 & -\\

Year & \textbf{100} & 43 & 45 & 42 & 40 & 42 & 74\\

%\midrule
		
%Average & \textbf{0.5241} & 0.2877 & 0.3531 & 0.3601 & 0.4619 &  & 0.4091\\

\bottomrule
\end{tabular}
\caption{[Higher is better] Average quality, measured by $\kl$, of top 50 subgroups. The results are relative and the quality of \kflexi on each data set is the base (100\%).} \label{tab:kl} 
\end{table}

\subsection{Quality results with $\qr$}

We recall that $\qr$ is suited to univariate/multivariate numeric and/or ordinal targets. In this experiment, we focus on multivariate targets; hence, \smdl and \RocInt are inapplicable. Regarding the setup, for Adult we combine the ordinal attribute and two randomly selected numeric attributes to form targets. For Cover, we pick three ordinal attributes as targets. For Bank, we combine the two ordinal attributes and two randomly selected numeric attributes to create targets. For Network, we randomly sample five ordinal attributes and five numeric attributes to form targets. For Drive and Year, we randomly pick half of the numeric attributes as targets.

To avoid runtimes of more than 5 hours on 
Cover, Network, Drive, and Year, for \textit{all} methods we subsample with $\epsilon = 0.1$. Note that with \ef, \ew, and \ipd, we need to compute subgroup quality after the bins have been formed, which in total is quadratic to the data size $\size$. Thus, to ensure efficiency subsampling is also necessary. For the final subgroups, we use their actual quality for evaluation. Each quality value resulted from using subsampling is the average of 10 runs; standard deviation is small and hence skipped.

The results are in Table~\ref{tab:qr}. We see that \qflexi outperforms all competitors with large margins, improving quality up to 14 times.

\begin{table}[t]
\centering 
\begin{tabular}{lrrrrr}
\toprule
{\bf Data} & {\bf \qflexi} & {\bf \alter} & {\bf \ef} & {\bf \ew} & {\bf \ipd}\\
\otoprule

Adult & \textbf{100} & 18 & 7 & 8 & 23\\

%Bike & 1.77(8) & 0.49 & 0.61 & 0.69 & \\

Cover & \textbf{100} & 60 & 41 & 39 & 53\\

%Gesture & 3.25(-24) & 0.82 & 1.13 & 2.58 & \\

%Letter & 0.59(6) & 0.35 & 0.36 & 0.41 & \\

Bank & \textbf{100} & 31 & 47 & 59 & 66\\

%Naval & 0.57 & 0.18 & 0.21 & 0.26 & \\

Network & \textbf{100} & 48 & 69 & 64 & 56\\

%SatImage & 3.57(31) & 1.23 & 2.20 & 1.94 & \\

Drive & \textbf{100} & 62 & 41 & 59 & 66\\

%Turkiye & 1.03(10) & 1.03 & 1.03 & 1.03 & \\

Year & \textbf{100} & 26 & 27 & 21 & 55\\

\bottomrule
\end{tabular}
\caption{[Higher is better] Average quality, measured by $\qr$, of top 50 subgroups. The results are relative and the quality of \qflexi on each data set is the base (100\%).} \label{tab:qr} 
\end{table}

\subsection{Efficiency results}

We here compare the efficiency of methods that have an advanced way to form binary features; that is, for fairness we skip \ef and \ew. The relative runtime of all remaining methods are shown in Figures~\ref{fig:time_wracc}, \ref{fig:time_kl}, and~\ref{fig:time_qr}. The results of our methods in each case are the bases. We observe that we overall are faster than \RocInt. This could be attributed to the fact that we form initial bins before mining actual features. \RocInt in turn uses the original set of cut points and hence has a larger search space per attribute. We can also see that our methods have comparable runtime to \alter. While in theory \alter is more efficient than our method, it may unnecessarily form too many binary features per attribute, which potentially incurs higher runtime for the whole subgroup discovery process.

\begin{figure*}[tb]
\centering
\subfigure[Runtime with $\WRAcc$]
{{\includegraphics[width=0.3\textwidth]{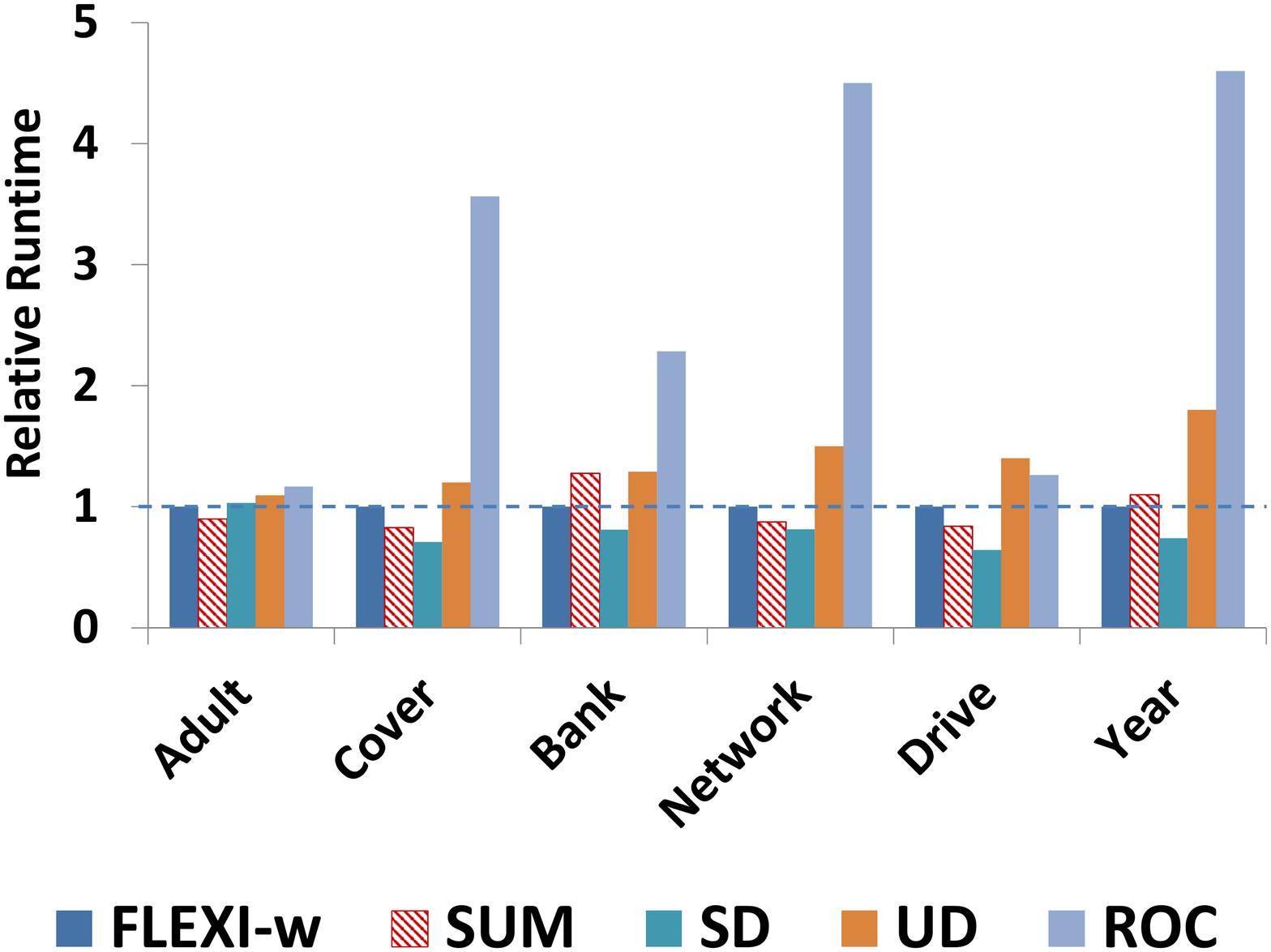}}\label{fig:time_wracc_short}}
\subfigure[Runtime with $\kl$]
{{\includegraphics[width=0.3\textwidth]{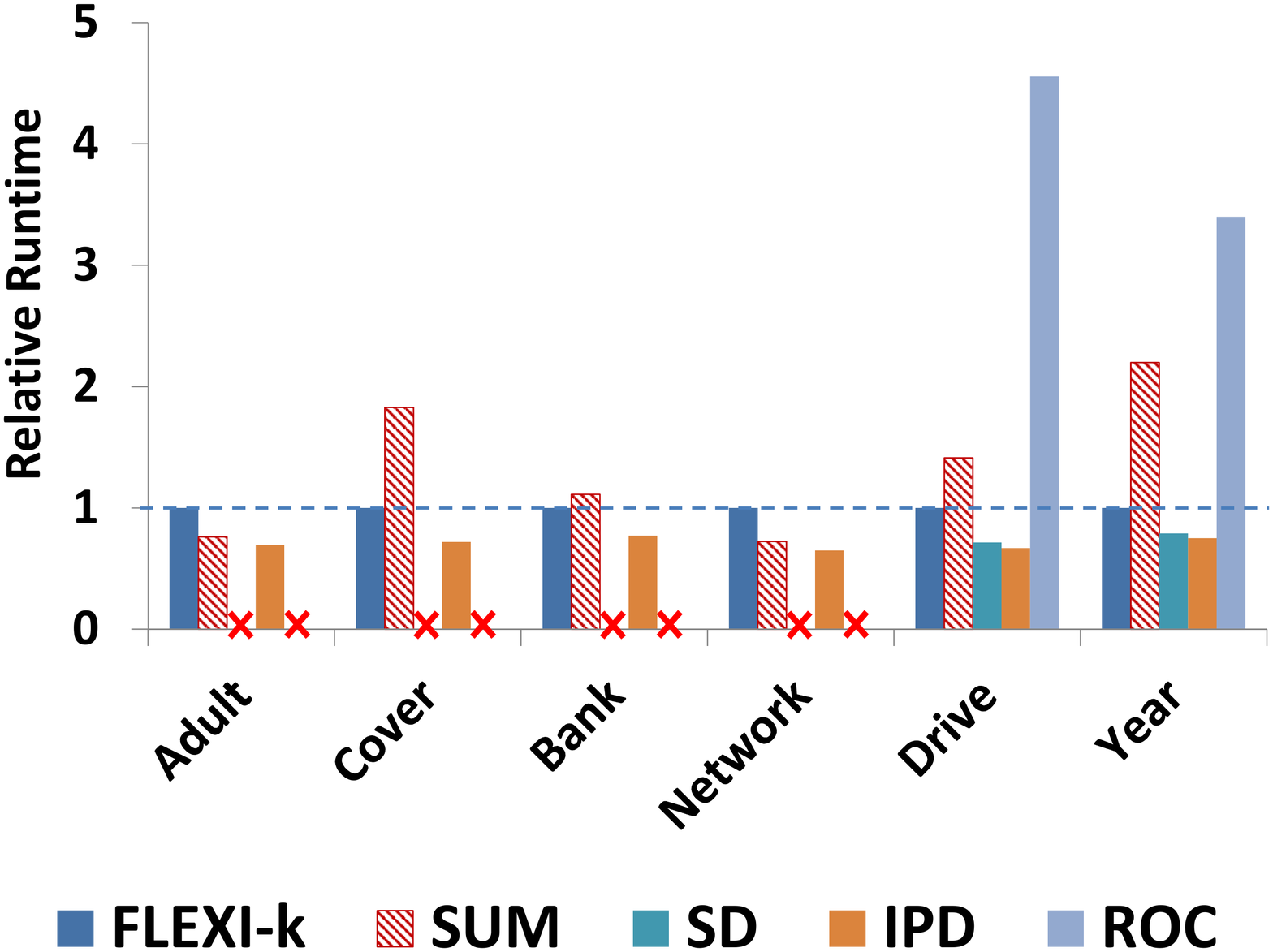}}\label{fig:time_kl_short}}
\subfigure[Runtime with $\qr$]
{{\includegraphics[width=0.3\textwidth]{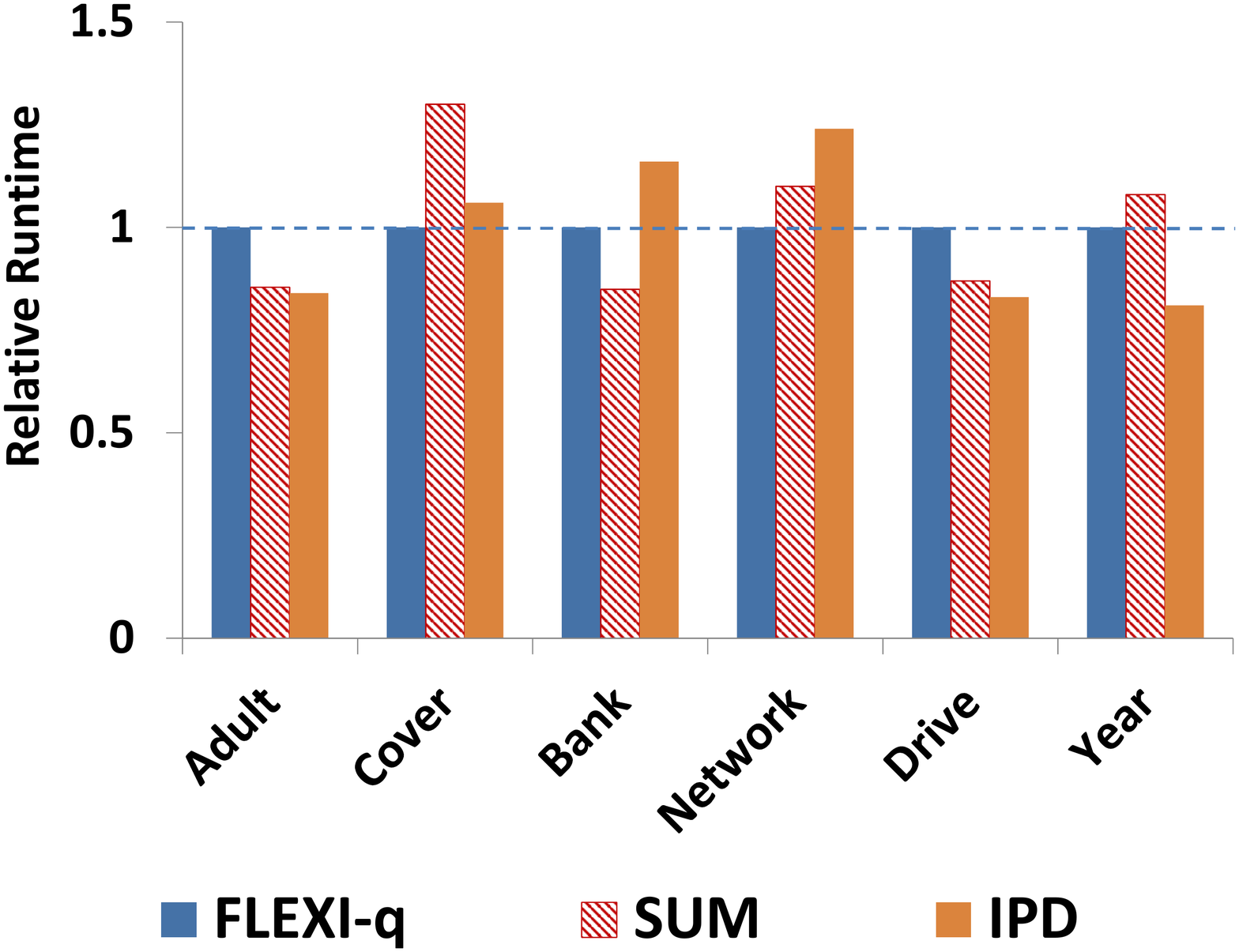}}\label{fig:time_qr_short}}
\caption{[Lower is better] Relative runtime with $\WRAcc$, $\kl$, and $\qr$. In each case the runtime of \ourmethod is the base. \smdl and \RocInt are not applicable to Adult, Cover, Bank, and Network, which is marked by {\bf \color{red}{\xmark}}.}
\end{figure*}

\subsection{Parameter setting}

\ourmethod has two input parameters: the number of initial bins $\beta$ and the subsampling rate $\epsilon$. To assess the sensitivity to $\beta$, we vary it from 5 to 40 with step size being 5. For sensitivity to $\epsilon$, we vary it from 0.05 to 0.2 with step size being 0.05. The default setting is $\beta = 20$ and $\epsilon = 0.1$. The results are in Figures~\ref{fig:quality_vs_beta} and~\ref{fig:quality_vs_epsilon}. For $\beta$, we show representative outcome of \wflexi and \kflexi on Adult and Bank. For $\epsilon$, we show outcome of \qflexi on Network and Drive. We can see that our methods are very stable to parameter setting.
%This suggests that our method permits robust parametrization in practice.

\begin{figure}[tb]
\centering
\subfigure[Quality vs.\ $\beta$]
{{\includegraphics[width=0.23\textwidth]{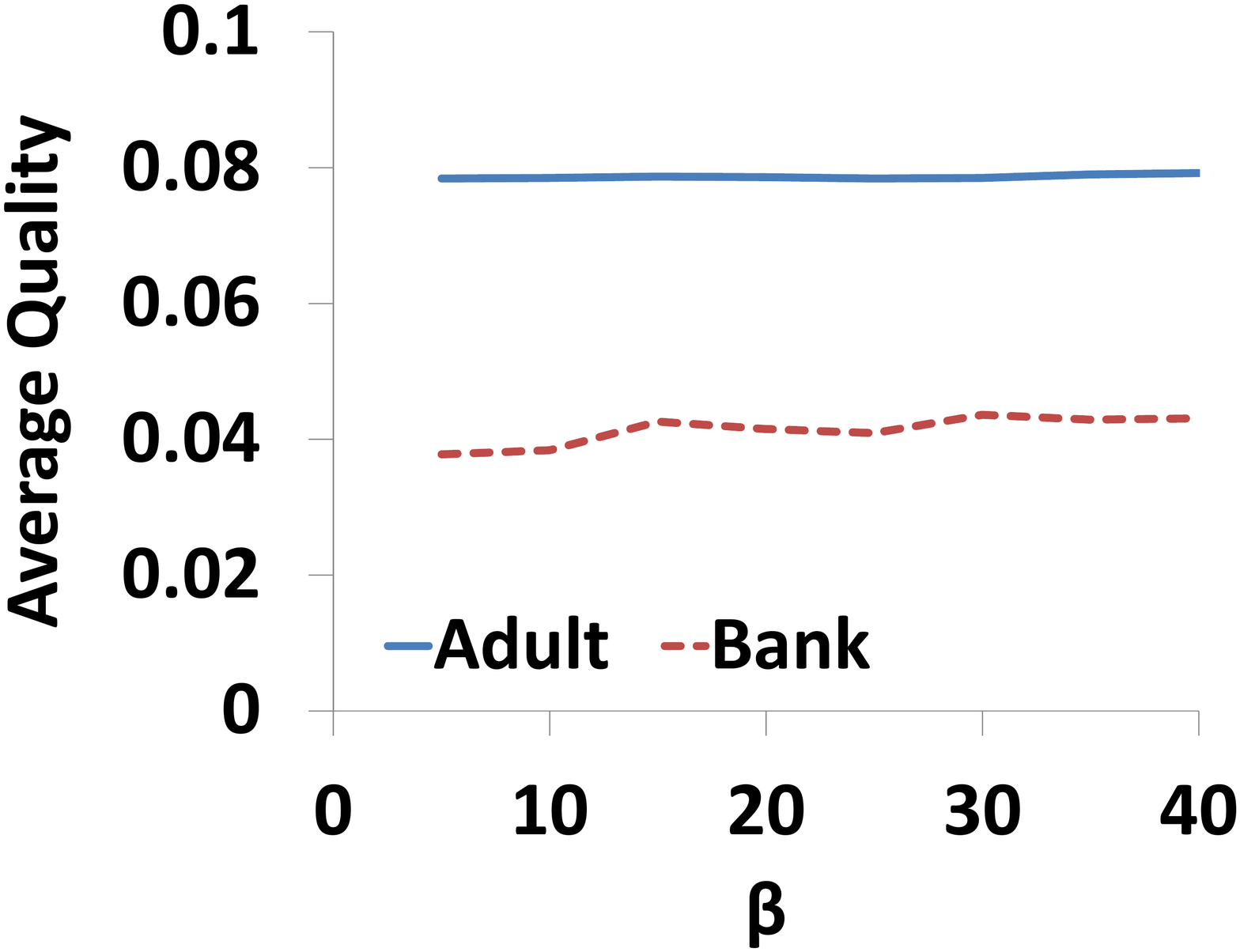}}\label{fig:quality_vs_beta}}
\subfigure[Quality vs.\ $\epsilon$]
{{\includegraphics[width=0.23\textwidth]{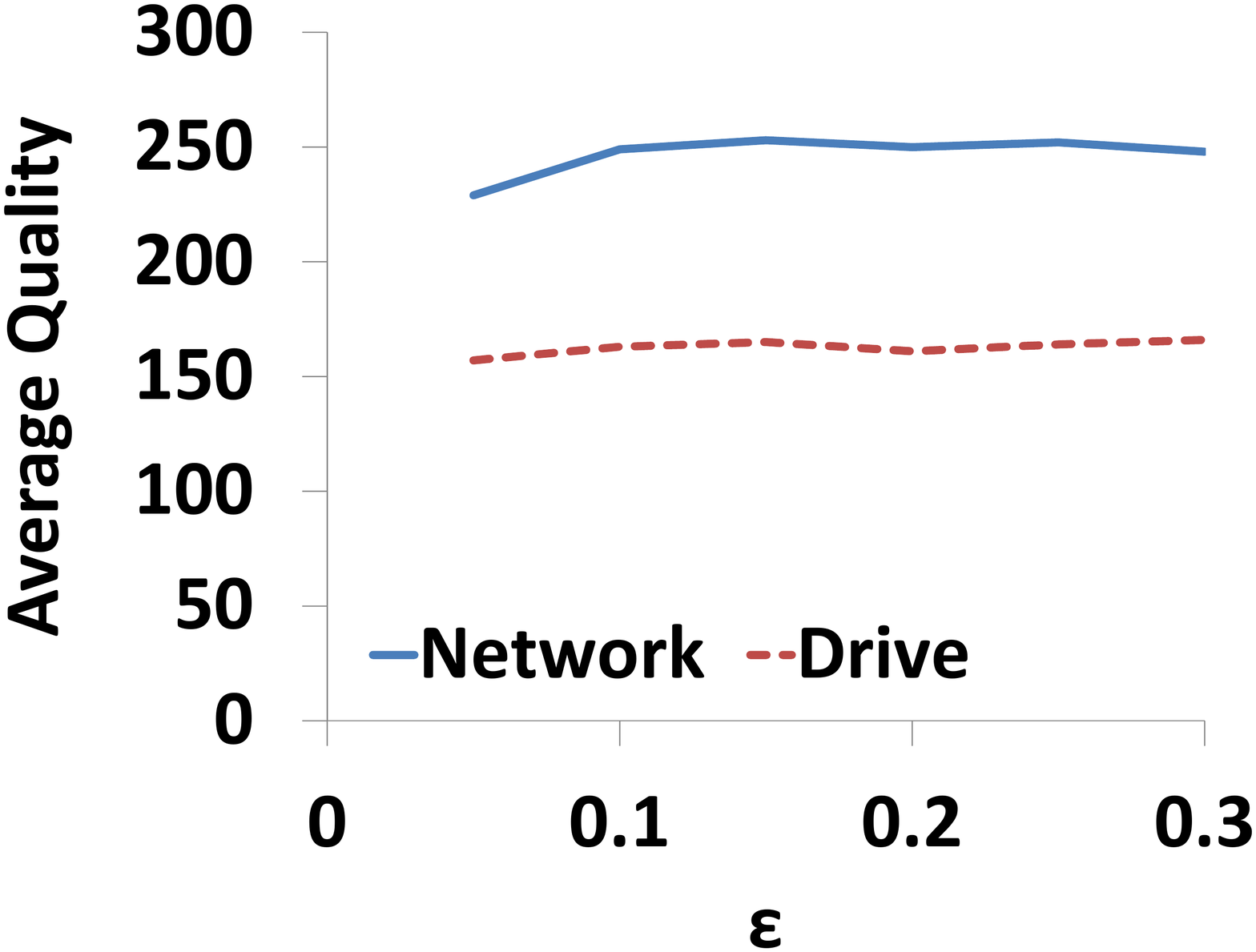}}\label{fig:quality_vs_epsilon}}
\caption{Sensitivity to $\beta$ and $\epsilon$. For $\beta$, we show the results of \wflexi and \kflexi on Adult and Bank. For $\epsilon$, we show the results of \qflexi on Network and Drive.}
\end{figure}

\section{Discussion} \label{sec:dis}

The experiments on different quality measures and real-world data sets show that \ourmethod found subgroups of higher quality than existing methods. In terms of efficiency, it is on par with \alter and faster than \RocInt -- the state of the art for mining binary features for subgroup discovery. The good performance of \ourmethod could be attributed to (1) our formulation of binary feature mining which takes into account subgroup quality, (2) our efficient dynamic programming algorithm which searches for optimal binary features, and (3) our subsampling method to handle very large data sets.

Yet, there is room for alternative methods as well as further improvements. For instance, in addition to beam search it is also interesting to apply \ourmethod to other search paradigms, e.g.\ MDL-based search~\cite{leeuwen:subgroup1}. Along this line, we can also formulate our search problem as mining binary features with high quality that together effectively compress the data. Besides the already demonstrated efficiency of our method, it can be further sped up by parallelization, e.g.\ with MapReduce. This direction in fact is applicable to subgroup discovery in general and is a potential solution towards making methods in this area more applicable to real-world scenarios.

\section{Conclusion} \label{sec:con}

We studied the problem of mining binary features for subgroup discovery. This is challenging as one needs a formulation that allows us to identify features leading to the detection of high quality subgroup. Second, the solution should place no restrictions on the target. Third, it should permit efficient computation. To address these issues, we proposed \ourmethod. In short, \ourmethod aims at identifying binary features per attribute with maximal average quality. The formulation of \ourmethod is abstract from the targets and hence suited to any type of targets. We instantiated \ourmethod with five different measures and showed how to make it efficient in every case. Extensive experiments on various real-world data sets verified that compared to existing methods, \ourmethod is able to efficiently detect subgroups with considerably higher quality.

\section*{Acknowledgements}
The authors are supported by the Cluster of Excellence ``Multimodal Computing and Interaction'' within the Excellence Initiative of the German Federal Government.

% Bibliography
\bibliographystyle{abbrv}
\bibliography{bib/abbrev,bib/citation,bib/bib-jilles}

\ifapx
\appendix
\section{Alternative Setting} \label{sec:alternative}

Here we show that the alternate problem formulation can also be solved by dynamic programing. More specifically, let $\dsc$ be the optimal solution and $\{b_\dsc^1, \ldots, b_\dsc^{|\dsc|}\}$ be its bins. It holds that
$$\sum_{i=1}^{|\dsc|} \phi(b_\dsc^i) = \phi(b_\dsc^{|\dsc|}) + \sum_{i=1}^{|\dsc| - 1} \phi(b_\dsc^i).$$
As $\dsc$ is optimal, $\{b_\dsc^1, \ldots, b_\dsc^{|\dsc|-1}\}$ must be the optimal binning for values $A \leq l_\dsc^{|\dsc|}$. Otherwise, we could have chosen a different binning for such values that improves the total quality. This would yield another binning for all values of $A$ that has a total quality higher than that of $\dsc$, which contradicts the assumption on $\dsc$. Hence, the optimal binning $\dsc$ also exhibits optimal substructure, permitting the use of dynamic programming. The detailed solution is in Algorithm~\ref{algo:one}.

\begin{algorithm}[t]
\caption{\textsc{\alter}}
\label{algo:one}
\begin{algorithmic}[1]
\STATE Create initial disjoint bins $\{c_1, \ldots, c_\beta\}$ of $A$

\STATE Create a double array $\qual[1\ldots\beta]$

\STATE Create an array $b[1\ldots\beta]$ whose each entry stores bins

\STATE Set $\qual[1] = \phi(c_1)$ and $b[1] = c_1$

\FOR{$i = 2 \rightarrow \beta$}
	\STATE $\pos = \arg\max\limits_{1 \leq j \leq i-1} \qual[j] + \phi(\bigcup_{k=j+1}^{i} c_k)$
	
	\STATE $\qual[i] = \qual[\pos] + \phi(\bigcup_{k=\pos+1}^{i} c_k)$
	
	\STATE Copy all bins in $b[\pos]$ to $b[i]$
	
	\STATE Add $\bigcup_{k=\pos+1}^{i} c_k$ to $b[i]$
\ENDFOR

\STATE Return $b[\beta]$
\end{algorithmic}
\end{algorithm}

%\section{Pre-computation with $\zscore$} \label{sec:zscorealgo}

%We give the details for the pre-computation of $\zscore(\bigcup_{k=j}^i c_k)$ for $1 \leq j \leq i \leq \beta$ in Algorithm~\ref{algo:zscore}.

\section{Proofs} \label{sec:proofs}

\begin{proof}[Theorem~\ref{theo:quadem}]
W.l.o.g., we assume that $\T_1, \ldots, \T_l$ are numeric and $\T_{l+1}, \ldots, \T_\dimt$ are ordinal. We have
\begin{align*}
& P(t_1, \ldots, t_d) =\\
& \int_{\minv_1}^{\maxv_1} \ldots \int_{\minv_l}^{\maxv_l} \sum\limits_{t_{l+1} \in \dom(T_{l+1})} \ldots \sum\limits_{t_\dimt \in \dom(T_\dimt)} \mathbf{I}(x_1 \leq t_1) \times \\
&  \cdots \times \mathbf{I}(x_\dimt \leq t_\dimt) \times p(x_1, \ldots, x_\dimt) dx_1 \cdots dx_\dimt.
\end{align*}
Similarly, we have
\begin{align*}
& P_\Sub(t_1, \ldots, t_d) =\\
& \int_{\minv_1}^{\maxv_1} \ldots \int_{\minv_l}^{\maxv_l} \sum\limits_{t_{l+1} \in \dom(T_{l+1})} \ldots \sum\limits_{t_\dimt \in \dom(T_\dimt)} \mathbf{I}(x_1 \leq t_1) \times \\
&  \cdots \times \mathbf{I}(x_\dimt \leq t_\dimt) \times p_\Sub(x_1, \ldots, x_\dimt) dx_1 \cdots dx_\dimt.
\end{align*}
Using empirical data, we have 
$$P(t_1, \ldots, t_d) = \frac{1}{\size} \sum_{i=1}^\size \prod_{k=1}^{\dimt} \mathbf{I}(\D^i_k \leq t_i), \quad  \textrm{ and }$$
$$P_\Sub(t_1, \ldots, t_d) = \frac{1}{s} \sum_{i=1}^s \prod_{k=1}^{\dimt} \mathbf{I}(\D^i_{\Sub,k} \leq t_i).$$
Hence, we have
\begin{align*}
& \id(p_\Sub(\Tb)\; ||\; p(\Tb)) =\\
& \int_{\minv_1}^{\maxv_1} \ldots \int_{\minv_l}^{\maxv_l} \sum\limits_{t_{l+1} \in \dom(T_{l+1})} \ldots \sum\limits_{t_\dimt \in \dom(T_\dimt)} \\
&  \left(\frac{1}{s} \sum_{i=1}^s \prod_{k=1}^{\dimt} \mathbf{I}(\D^i_{\Sub,k} \leq t_i) - \frac{1}{\size} \sum_{i=1}^\size \prod_{k=1}^{\dimt} \mathbf{I}(\D^i_k \leq t_i)\right)^2 dt_1 \cdots dt_l.
\end{align*}
Expanding the above term and bringing the integrals inside the sums, we have
\begin{align*}
& \id(p_\Sub(\Tb)\; ||\; p(\Tb)) =\\
& \frac{1}{s^2} \sum_{i=1}^s \sum_{j=1}^s \left(\prod_{k=1}^{l} \int_{\minv_i}^{\maxv_i} \mathbf{I}(\max(\D_{\Sub,k}^i, \D_{\Sub,k}^j) \leq t_k) dt_k\right) \times\\
& \qquad \qquad \left(\prod_{k=l+1}^{\dimt} \sum\limits_{t_k \in \dom(T_k)} \mathbf{I}(\max(\D_{\Sub,k}^i, \D_{\Sub,k}^j) \leq t_k)\right)\\
& - \frac{2}{s \size} \sum_{i=1}^s \sum_{j=1}^\size \left(\prod_{k=1}^{l} \int_{\minv_i}^{\maxv_i} \mathbf{I}(\max(\D_{\Sub,k}^i, \D_k^j) \leq t_k) dt_k\right) \times \\
& \qquad \qquad \left(\prod_{k=l+1}^{\dimt} \sum\limits_{t_k \in \dom(T_k)} \mathbf{I}(\max(\D_{\Sub,k}^i, \D_k^j) \leq t_k)\right)\\
& + \frac{1}{\size^2} \sum_{i=1}^\size \sum_{j=1}^\size \left(\prod_{k=1}^{l} \int_{\minv_i}^{\maxv_i} \mathbf{I}(\max(\D_k^i, \D_k^j) \leq t_k) dt_k\right) \times\\
& \qquad \qquad \left(\prod_{k=l+1}^{\dimt} \sum\limits_{t_k \in \dom(T_k)} \mathbf{I}(\max(\D_k^i, \D_k^j) \leq t_k)\right)
\end{align*}
by which we arrive at the final result.
\end{proof}

\begin{proof}[Lemma~\ref{lem:quadcom}]
Empirically, we have that
\begin{align*}
& \diff(p_{\Sub \cup \Sube}(\Tb)\; ||\; p(\Tb))\\
&= \frac{1}{(s+r)^2} \sum_{i=1}^{s+r}\sum_{j=1}^{s+r} \prod_{k=1}^{\dimt} h_k(\D_{\Sub \cup \Sube,i}^k, \D_{\Sub \cup \Sube,j}^k) \\
&- \frac{2}{(s+r) \size} \sum_{i=1}^{s+r}\sum_{j=1}^{\size} \prod_{k=1}^{\dimt} h_k(\D_{\Sub \cup \Sube,i}^k, \D_j^k) \\ 
&+ \frac{1}{\size^2} \sum_{i=1}^{\size}\sum_{j=1}^{\size} \prod_{k=1}^{\dimt} h_k(\D_i^k, \D_j^k).
\end{align*}
We can see that the first term is equal to $\frac{1}{(s+r)^2} \Sub.e_1 + \frac{1}{(s+r)^2} \Sube.e_1 + \frac{2}{(s + r)^2} \inter(\Sub, \Sube)$ where $\inter(\Sub, \Sube) = \sum\limits_{i=1}^s \sum\limits_{j=1}^r \prod\limits_{k=1}^\dimt h_k(\D_{\Sub,i}^k, \D_{\Sube,j}^k)$. The second term is equal to $\frac{2}{(s+r) \size} \Sub.e_2 + \frac{2}{(s+r) \size} \Sube.e_2$. The third term is in fact $e$.
\end{proof}

\begin{proof}[Lemma~\ref{lem:inter}]
By definition, we have that
\begin{align*}
&\inter\left(\Sub, \bigcup_{i=1}^l \Sube_i\right)\\
&= \sum_{q=1}^s \sum_{j=1}^{s_1 + \ldots + s_l} \prod_{k=1}^\dimt h_k(\D_{\Sub,q}^k, \D_{\bigcup_{i=1}^l \Sube_i,j}^k)\\
&= \sum_{i=1}^l \sum_{q=1}^s \sum_{j=1}^{s_i} \prod_{k=1}^\dimt h_k(\D_{\Sub,q}^k, \D_{\Sube_i,j}^k)\\
&= \sum_{i=1}^l \inter(\Sub, \Sube_i).
\end{align*}
\end{proof}

\section{Additional Experimental Results} \label{sec:full}

Quality results on all quality measures are in Tables~\ref{tab:wraccfull}, \ref{tab:zscorefull}, \ref{tab:klfull}, \ref{tab:hdfull}, and~\ref{tab:qrfull}. Note that we show absolute values. As Naval has neither categorical nor ordinal attributes, it is not applicable to $\WRAcc$, $\kl$, and $\hd$.

Additional efficiency results are in Figures~\ref{fig:time_wracc}, \ref{fig:time_kl}, and~\ref{fig:time_qr}. Interestingly, on $\qr$ measure, \qflexi is even faster than \ew on 3 data sets. Our explanation is similar to the case of \alter; that is, \ew may form unnecessarily many binary features than required per attribute which prolongs the runtime.

\begin{table*}[t]
\centering 
\begin{tabular}{lrrrrrrr}
\toprule
{\bf Data} & {\bf \wflexi} & {\bf \alter} & {\bf \ef} & {\bf \ew} & {\bf \smdl} & {\bf \unml} & {\bf \RocInt}\\
\otoprule

Adult & \textbf{0.08} & 0.07 & 0.07 & 0.07 & 0.07 & 0.06 & 0.07\\

Bike & \textbf{0.06} & 0.04 & 0.04 & 0.04 & 0.06 & 0.04 & 0.05\\

Cover & \textbf{0.12} & 0.11 & 0.04 & 0.08 & 0.04 & 0.05 & 0.04\\

Gesture & \textbf{0.10} & 0.08 & 0.03 & 0.09 & 0.07 & 0.04 & 0.04\\

Letter & \textbf{0.08} & 0.05 & 0.02 & 0.03 & 0.05 & 0.04 & 0.04\\

Bank & \textbf{0.04} & 0.03 & 0.02 & 0.03 & 0.02 & 0.02 & 0.02\\

Network & \textbf{0.18} & 0.13 & 0.10 & 0.12 & 0.14 & 0.12 & 0.14\\

SatImage & \textbf{0.15} & 0.11 & 0.03 & 0.05 & 0.09 & 0.04 & 0.05\\

Drive & \textbf{0.11} & 0.08 & 0.03 & 0.08 & 0.05 & 0.06 & 0.05\\

Turkiye & \textbf{0.11} & \textbf{0.11} & 0.10 & 0.10 & 0.10 & 0.10 & 0.10\\

Year & \textbf{0.12} & 0.08 & 0.06 & 0.06 & 0.07 & 0.06 & 0.07\\

\midrule

Average & \textbf{0.10} & 0.08 & 0.05 & 0.07 & 0.07 & 0.06 & 0.06\\

\bottomrule
\end{tabular}
\caption{[Higher is better] Average quality, measured by $\WRAcc$, of top 50 subgroups. Best values are in \textbf{bold}.} \label{tab:wraccfull} 
\end{table*}

\begin{table*}[t]
\centering 
\begin{tabular}{lrrrrrr}
\toprule
{\bf Data} & {\bf \zflexi} & {\bf \alter} & {\bf \ef} & {\bf \ew} & {\bf \unml} & {\bf \RocInt}\\
\otoprule

Adult & \textbf{89.44} & 82.14 & 82.14 & 86.04 & 79.62 & 82.14\\

Bike & \textbf{68.61} & 50.44 & 57.54 & 50.24 & 56.25 & 61.50\\

Cover & \textbf{434.97} & 328.43 & 356.44 & 249.49 & 288.29 & 384.48\\

Gesture & 38.09 & 31.38 & 35.32 & 33.55 & 31.42 & \textbf{44.01}\\

Letter & \textbf{47.11} & 41.90 & 43.82 & 39.97 & 40.77 & 44.17\\

Bank & \textbf{78.76} & 69.54 & 72.45 & 71.39 & 66.40 & 72.45\\

Naval & 28.20 & 23.60 & 22.92 & 22.50 & 22.61 & \textbf{32.25}\\

Network & 135.09 & 129.38 & 133.60 & 114.78 & 110.45 & \textbf{145.91}\\

SatImage & \textbf{50.28} & 35.23 & 39.32 & 41.94 & 39.42 & 44.16\\

Drive & \textbf{120.33} & 86.64 & 69.57 & 46.93 & 44.43 & 40.80\\

Turkiye & \textbf{14.56} & 9.53 & 9.53 & 9.54 & 7.10 & 12.37\\

Year & \textbf{88.57} & 57.59 & 47.93 & 53.50 & 50.40 & 60.31\\

\midrule
		
Average & \textbf{99.50} & 78.82 & 80.88 & 68.32 & 69.76 & 85.38\\

\bottomrule
\end{tabular}
\caption{[Higher is better] Average quality, measured by $\zscore$, of top 50 subgroups. Best values are in \textbf{bold}.} \label{tab:zscorefull} 
\end{table*}

\begin{table*}[t]
\centering 
\begin{tabular}{lrrrrrrr}
\toprule
{\bf Data} & {\bf \kflexi} & {\bf \alter} & {\bf \ef} & {\bf \ew} & {\bf \smdl} & {\bf \ipd} & {\bf \RocInt}\\
\otoprule

Adult & \textbf{0.52} & 0.20 & 0.19 & 0.16 & \emph{n/a} & 0.02 & \emph{n/a}\\

Bike & \textbf{0.50} & 0.26 & 0.34 & 0.35 & \emph{n/a} & 0.05 & \emph{n/a}\\

Cover & \textbf{0.53} & 0.23 & 0.34 & 0.40 & \emph{n/a} & 0.24 & \emph{n/a}\\

Gesture & \textbf{0.53} & 0.22 & 0.33 & 0.33 & 0.50 & 0.31 & 0.33\\

Letter & \textbf{0.52} & 0.43 & 0.43 & 0.47 & 0.43 & 0.06 & 0.43\\

Bank & \textbf{0.52} & 0.24 & 0.32 & 0.17 & \emph{n/a} & 0.03 & \emph{n/a}\\

Network & \textbf{0.53} & 0.29 & 0.36 & 0.29 & \emph{n/a} & 0.11 & \emph{n/a}\\

SatImage & \textbf{0.53} & 0.28 & 0.37 & 0.48 & 0.45 & 0.26 & 0.37\\

Drive & \textbf{0.53} & 0.22 & 0.34 & 0.45 & 0.47 & 0.22 & 0.33\\

Turkiye & \textbf{0.53} & 0.50 & 0.50 & 0.50 & \emph{n/a} & 0.15 & \emph{n/a}\\

Year & \textbf{0.53} & 0.23 & 0.24 & 0.22 & 0.21 & 0.22 & 0.39\\

\midrule
		
Average & \textbf{0.52} & 0.28 & 0.34 & 0.35 & 0.19 & 0.16 & 0.17\\

\bottomrule
\end{tabular}
\caption{[Higher is better] Average quality, measured by $\kl$, of top 50 subgroups. Best values are in \textbf{bold}.} \label{tab:klfull} 
\end{table*}

\begin{table*}[t]
\centering 
\begin{tabular}{lrrrrrrr}
\toprule
{\bf Data} & {\bf \hflexi} & {\bf \alter} & {\bf \ef} & {\bf \ew} & {\bf \smdl} & {\bf \ipd} & {\bf \RocInt}\\
\otoprule

Adult & \textbf{0.29} & 0.26 & 0.26 & 0.26 & \emph{n/a} & 0.22 & \emph{n/a}\\

Bike & \textbf{0.27} & 0.10 & 0.14 & 0.22 & \emph{n/a} & 0.25 & \emph{n/a}\\

Cover & \textbf{0.30} & \textbf{0.30} & 0.22 & 0.21 & \emph{n/a} & 0.27 & \emph{n/a}\\

Gesture & 0.29 & 0.08 & 0.14 & \textbf{0.30} & 0.27 & \textbf{0.30} & 0.14\\

Letter & \textbf{0.29} & 0.21 & 0.21 & 0.25 & 0.24 & 0.25 & 0.28\\

Bank & \textbf{0.29} & 0.13 & 0.16 & 0.23 & \emph{n/a} & 0.26 & \emph{n/a}\\

Network & \textbf{0.29} & 0.22 & 0.22 & 0.21 & \emph{n/a} & 0.25 & \emph{n/a}\\

SatImage & \textbf{0.29} & 0.11 & 0.16 & 0.24 & 0.23 & 0.23 & 0.17\\

Drive & 0.29 & 0.08 & 0.14 & 0.28 & 0.29 & \textbf{0.30} & 0.14\\

Turkiye & \textbf{0.29} & 0.26 & 0.26 & 0.26 & \emph{n/a} & 0.26 & \emph{n/a}\\

Year & \textbf{0.29} & 0.25 & 0.12 & 0.14 & 0.14 & 0.22 & 0.15\\

\midrule
		
Average & \textbf{0.29} & 0.18 & 0.18 & 0.24 & 0.11 & 0.26 & 0.08\\

\bottomrule
\end{tabular}
\caption{[Higher is better] Average quality, measured by $\hd$, of top 50 subgroups. Best values are in \textbf{bold}.} \label{tab:hdfull} 
\end{table*}

\begin{table*}[t]
\centering 
\begin{tabular}{lrrrrr}
\toprule
{\bf Data} & {\bf \qflexi} & {\bf \alter} & {\bf \ef} & {\bf \ew} & {\bf \ipd}\\
\otoprule

Adult & \textbf{110.35} & 20.1 & 8.19 & 8.58 & 25.38\\

Bike & \textbf{1.77} & 0.49 & 0.61 & 0.69 & 0.75\\

Cover & \textbf{185.72} & 110.51 & 76.58 & 71.95 & 98.52\\

Gesture & \textbf{3.25} & 0.82 & 1.13 & 2.58 & 2.86\\

Letter & \textbf{0.59} & 0.35 & 0.36 & 0.41 & 0.44\\

Bank & \textbf{41.71} & 13.02 & 19.60 & 24.54 & 27.63\\

Naval & \textbf{0.57} & 0.18 & 0.21 & 0.26 & 0.28\\

Network & \textbf{25.72} & 12.37 & 17.63 & 16.34 & 14.34\\

SatImage & \textbf{3.57} & 1.23 & 2.20 & 1.94 & 2.11\\

Drive & \textbf{6.37} & 3.94 & 2.64 & 3.76 & 4.22\\

Turkiye & \textbf{1.03} & 0.85 & 0.77 & 0.83 & 0.83\\

Year & \textbf{271.98} & 69.41 & 73.07 & 55.95 & 149.43\\

\midrule

Average & \textbf{54.39} & 19.44 & 16.92 & 15.65 & 27.23\\

\bottomrule
\end{tabular}
\caption{[Higher is better] Average quality, measured by $\qr$, of top 50 subgroups. Best values are in \textbf{bold}.} \label{tab:qrfull} 
\end{table*}

\begin{figure*}[tb]
\centering
\subfigure[Runtime with $\WRAcc$]
{{\includegraphics[width=0.3\textwidth]{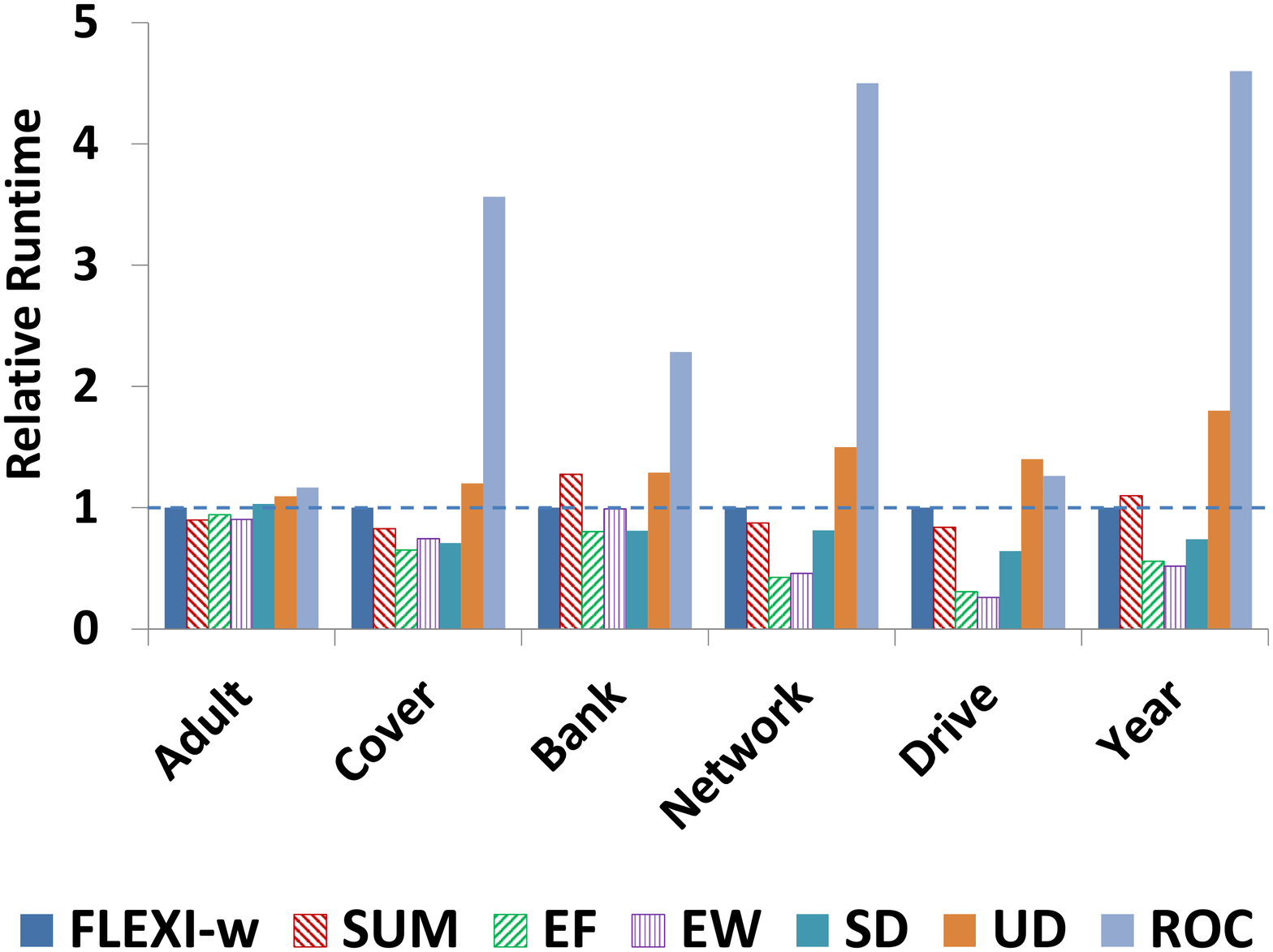}}\label{fig:time_wracc}}
\subfigure[Runtime with $\kl$]
{{\includegraphics[width=0.3\textwidth]{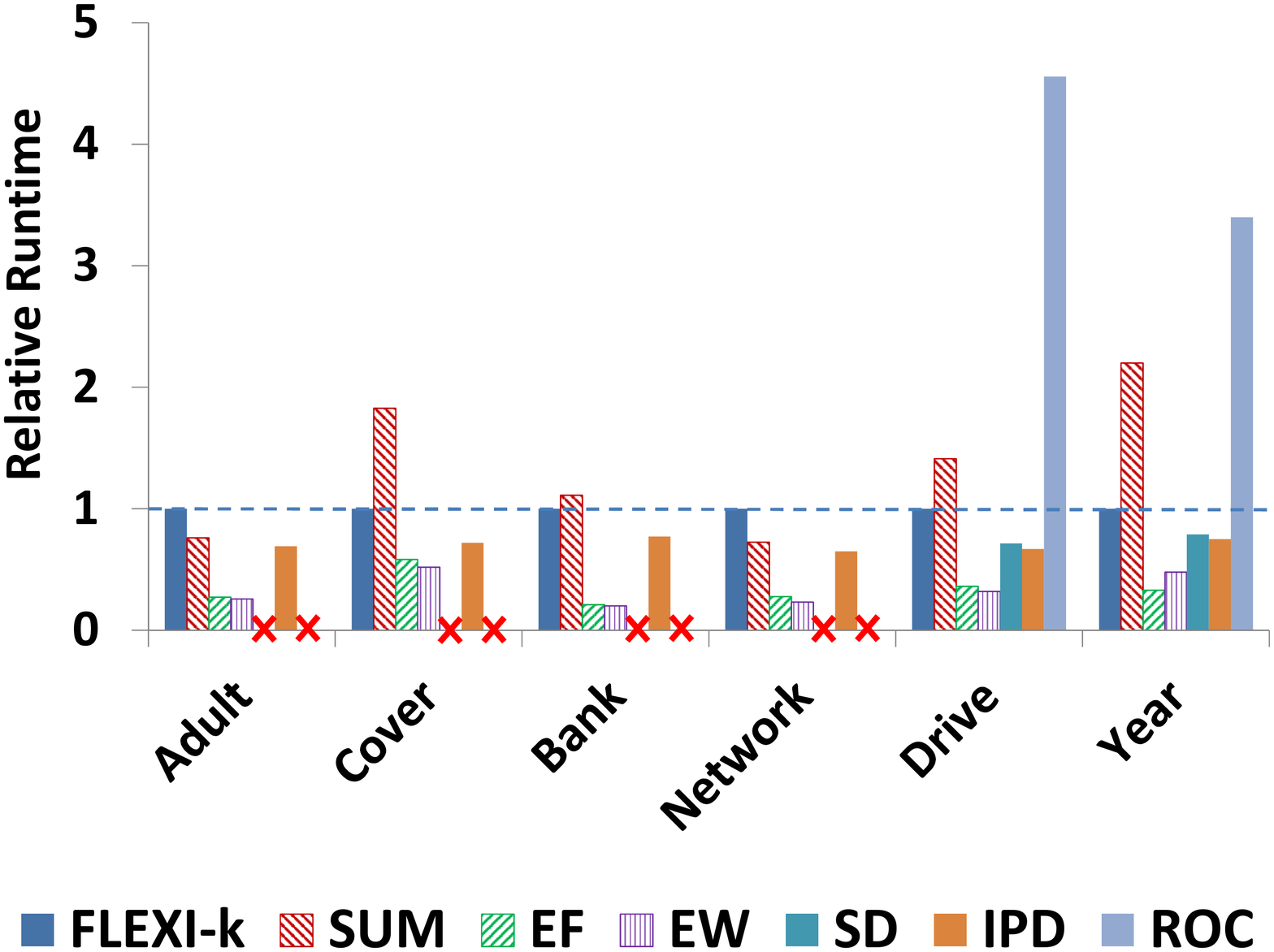}}\label{fig:time_kl}}
\subfigure[Runtime with $\qr$]
{{\includegraphics[width=0.3\textwidth]{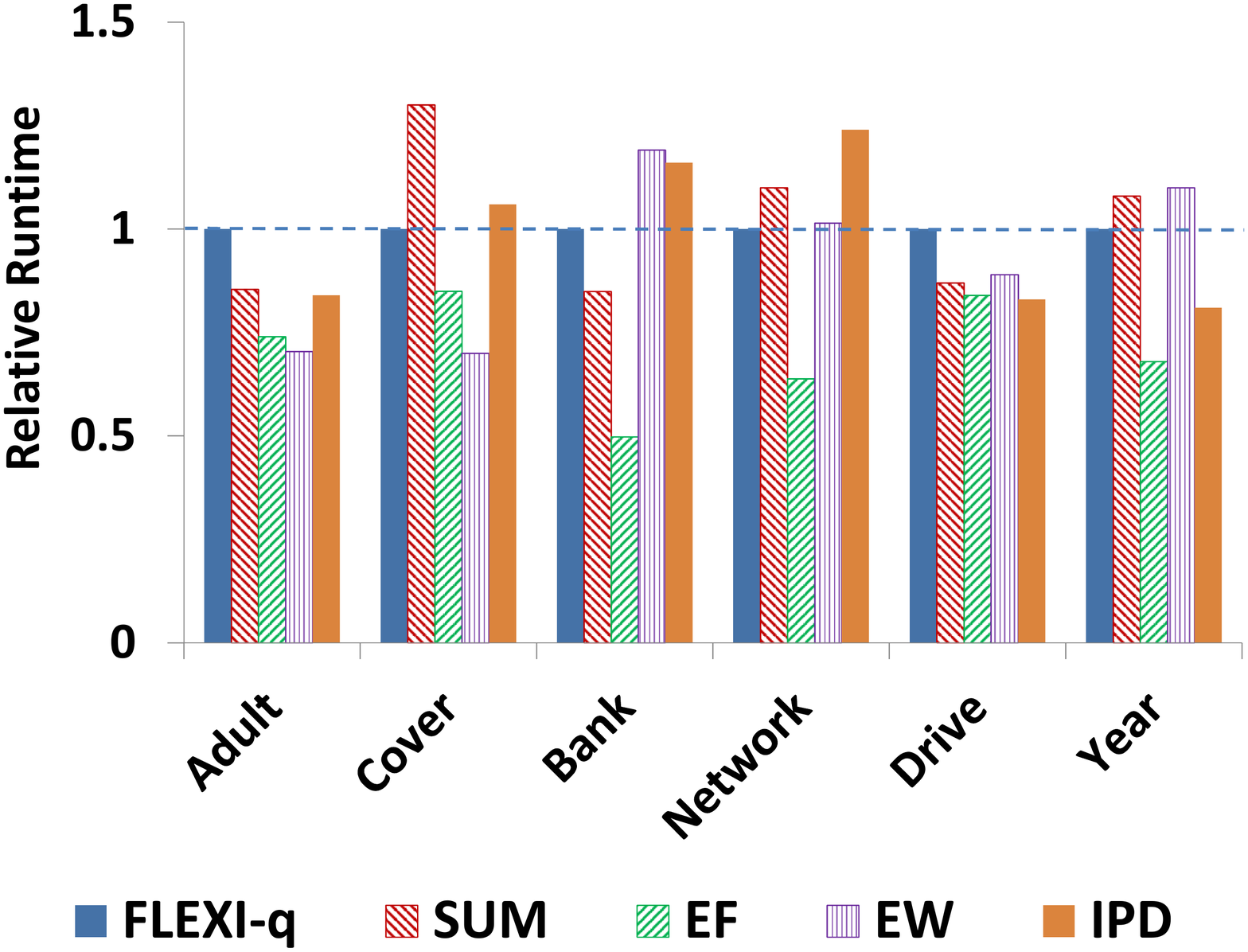}}\label{fig:time_qr}}
\caption{[Lower is better] Relative runtime with $\WRAcc$, $\kl$, and $\qr$. The runtime of our methods in each case is the base. \smdl and \RocInt are not applicable to Adult, Cover, Bank, and Network, which is marked by {\bf \color{red}{\xmark}}.}
\end{figure*}

\fi

\end{document}